\documentclass[10pt]{article}

\usepackage[utf8]{inputenc} 
\usepackage[T1]{fontenc}    
\usepackage{hyperref}       
\usepackage{url}            
\usepackage{booktabs}       
\usepackage{nicefrac}       
\usepackage{microtype}      

\usepackage{algorithmic}
\usepackage{graphicx}
\graphicspath{{figure/}}
\usepackage{textcomp}
\usepackage{fancyhdr}
\usepackage{algorithm}
\usepackage{epsfig}
\usepackage{hyperref}
\usepackage{multirow,array}
\usepackage{graphicx}
\usepackage{epstopdf}
\usepackage{booktabs}
\usepackage{longtable}
\usepackage{supertabular}

\usepackage{enumitem} 
\usepackage{amsmath, amssymb, bm, amsfonts} 
\usepackage{amsthm}

\def \X {\mathcal{X}}
\def \Y {\mathcal{Y}}
\def \U {\mathcal{U}}
\def \V {\mathcal{V}}
\def \D {\mathcal{D}}

\def \RR {\mathbb{R}}

\def \II {\mathbb{I}}

\usepackage[body={6.5in,8in}]{geometry}

\date{}

\author{
	Jiaqi Lv$^{1,2}$, Tianran Wu$^{1,2}$, Chenglun Peng$^{1,2}$, Yunpeng Liu$^{1,2}$, Ning Xu$^{1,2}$, Xin Geng$^{1,2}$\thanks{Corresponding author.}\\
	$^1$School of Computer Science and Engineering, Southeast University, Nanjing 210096, China\\
	$^2$Key Laboratory of Computer Network and Information Integration (Southeast University), \\Ministry of Education, China\\
	\texttt{\{lvjiaqi, trwu, pengchenglun, yunpengliu, xning, xgeng\}@seu.edu.cn}\\
}

\title{Compact Learning for Multi-Label Classification}

\begin{document}

\maketitle

\begin{abstract}
Multi-label classification (MLC) studies the problem where each instance is associated with multiple relevant labels, which leads to the exponential growth of output space. MLC encourages a popular framework named \emph{label compression} (LC) for capturing label dependency with dimension reduction.  Nevertheless, most existing LC methods failed to consider the influence of the feature space or misguided by original problematic features, so that may result in performance degeneration. In this paper, we present a \emph{compact learning} (CL) framework to embed the features and labels \emph{simultaneously and with mutual guidance}. The proposal is a versatile concept, hence the embedding way is arbitrary and independent of the subsequent learning process. Following its spirit, a simple yet effective implementation called \emph{compact multi-label learning} (CMLL) is proposed to learn a compact low-dimensional representation for both spaces. CMLL maximizes the dependence between the embedded spaces of the labels and features, and minimizes the loss of label space recovery concurrently. Theoretically, we provide a general analysis for different embedding methods. Practically, we conduct extensive experiments to validate the effectiveness of the proposed method.
\end{abstract}

\section{Introduction}\label{sec:intro}

\emph{Multi-label classification} (MLC) \cite{a2014zhang} is one of the mostly-studied machine learning paradigms, owing its popularity to its capability to fit the pervasive real-world tasks. It allows each instance to be equipped with multiple relevant labels for explicitly expressing the rich semantic meanings simultaneously. Nowadays MLC has been the prime focus due to its vast potential applications such as image annotation \cite{liu2017semantic, liu2018svm}, face recognition \cite{zhao2016joint, zhuang2018multi}, text categorization \cite{liu2017deep, zhang2018multi}, etc. 

Formally speaking, let $\X\subseteq\RR^D$ be the instance space and $\Y=[M]$ be the label space, where $D$ is the feature space dimension, $[M]:=\{1,2,\ldots,M\}$ and $M>2$ is the number of classes. The multi-label training set is represented as $\D=\{(\bm x_i,Y_i)\in\X\times\Y\}_{i=1}^N$ consisting of a $D$-dimensional instance $\bm x_i\in\X$ and the associated label set $Y_i\subseteq\Y$. MLC aims to induce a multi-label classifier $g:\X\rightarrow\Y$ to assign a set of relevant labels for the unseen instance. 

It is evident that MLC can be regarded as a generalization of traditional single-label learning. However, the generality inevitably leads to the output space grows exponentially as the number of classes increases. Inevitably, the curse of dimensionality in the label spaces becomes one of the major concerns in MLC, which results in many algorithms in low-dimensional space being ineffectiveness. A prominent phenomenon is the \emph{sparsity} of the label space, including \emph{label-set sparsity} and \emph{hypercube sparsity} \cite{multi2012tai}. The former means that the instances are usually associated with very few relevant labels compared to the label dimensionality, while the latter means that compared to all possible label combinations (i.e., the power sets of the label space), only a few are covered by the limited training data. For example, suppose an image annotation task with $50$ candidate labels, an image could often be related to no more than ten objects (i.e., label-set sparsity) and the collected training examples is far less than $2^{50}$ possible label combinations (i.e., hypercube sparsity). When the label space is considerably large, most of the conventional MLC algorithms become computationally inefficient, let alone tends to be corrupted by noisy labeling \cite{wang2019ranking}. Therefore, exponentialsized output space is still one of the major challenges for MLC methods.

There are many attempts to address the challenge, where \emph{label compression} (LC) \cite{multi2009hus,rai2015large,learning2017yeh,kumar2019group} is the dominant strategy. LC embeds the original high-dimensional label space into a low-dimensional subspace so as to gain a tighter label representation, followed by the association between the feature space and embedded label space for the learning purpose. By compression, problems such as redundancy and sparsity existing in the original label space can be alleviated to some extent, and also reduce the computational and space complexities.

While almost all LC methods only concentrate on the embedding of labels while keeping the features unchanged. Most of them \cite{multi2009hus,multi2012tai,multi2012wicker,multi2015li} totally ignored the influence of the features, which causes the loss of discriminant information and hurts the specific classification purposes. There are only few initial efforts in jointly utilizing the feature space for LC \cite{feature2012chen,multi2014lin}. However, how to learn the representative and discriminative features is still challenging, therefore problems such as noise and redundancy may also exist in the original features and mislead the label embedding. Correspondingly, some supervised \emph{feature embedding} (FE) methods \cite{canonical2003hardoon,multi2010zhang} have been proposed to focus on embedding the features into a new space. It is acknowledged that the importance of the same feature in different learning tasks may be inconsistent, thus the FE process should also be guided by the label space. The separate embedding of a single space driven by another problematic space provokes the propagation and accumulation of errors. 

In light of the above observation, in this paper, we focus on studying a general framework that co-embeds the two spaces in the MLC case. First we argue that:
\begin{itemize}[topsep=0ex,itemsep=-1ex]
	\item
	The embedding process of the label space and the feature space should be \emph{linked} to each other and performed \emph{simultaneously};
	\item
	The embedding process of one space should be guided by another \emph{well-disposed} space rather than the original problematic space.
\end{itemize}
Such a framework is named \emph{Compact Learning} (CL) in the sense that the learning is based on the two compact spaces. Then, to this end, we propose a simple yet effective algorithm called \emph{Compact Multi-Label Learning} (CMLL). CMLL aims to learn a more compact representation for both labels and features by maximizing the dependence between the embedded spaces of the labels and features, and simultaneously minimizing the recovery loss from the embedded labels to the original ones. In this way, the embedding processes of the two spaces are seamless and mutually guided, regardless of what methods are used in the learning process. We conduct comprehensive experiments over twelve benchmark datasets to validate the effectiveness of CMLL in improving the classification performance. 

The rest of the paper is organized as follows. Section~\ref{sec:related} briefly discusses the related work. Section~\ref{sec:cmll} presents the technical details of the proposed CMLL approach. Then, Section~\ref{sec:theory} conducts some theoretical analyses on LC and CL framework. Section~\ref{sec:exp} reports the experimental results. At last, Section~\ref{sec:conclude} concludes this paper.

\section{Related Work} \label{sec:related}

In this section, we mainly introduce the LC framework and review the important works in the field of LC. LC is a popular strategy for MLC where the target is to embed the original labels into a low-dimensional latent space. Generally speaking, LC consists of the following three processes:
\begin{enumerate}[topsep=0ex,itemsep=-1ex]
	\item Encoding/embedding process: It embeds the original label vectors into a compressed space through a specific transformation $e:\Y\to\V$, where $\V=[m] (m\ll M)$ is the embedded $m$-dimensional label space.
	\item Learning process. It induces a multi-label classifier from the feature space to the embedded space $g':\X\to\V$.
	\item
	Decoding/recovery process: It recovers the original labels from the embedded label space via a decoder $d:\V\to\Y$.
\end{enumerate}
For a new instance $\bm{x}$, the predicted labels $\widehat{Y}$ is: $\widehat{Y}=d(g'(\bm{x}))$. 

Most LC methods learned the embedding labels with the feature unchanged. \cite{multi2009hus} took one of the initial attempts to conduct LC via compressive sensing, which is time-consuming in the decoding process since it needs to solve an optimization problem for each new instance. Unlike compressive sensing, \cite{multi2012tai} proposed a principle label space transformation (PLST) method, which is essentially a principal components analysis in the label space. Then, based on canonical correlation analysis \cite{sun2010canonical}, Conditional PLST (CPLST) \cite{feature2012chen} and CCA-OC \cite{multi2011zhang} improved PLST from the point of feature information. \cite{dependence2015zhang} put forward a method to maximize the dependence between features and embedding labels. Some LC methods also apply the randomized techniques to speed up the computing \cite{joly2014random,fast2015mineiro}.

Instead of adopting the linear mapping, another kind of LC methods reduced the label-space dimensionality via a nonlinear mapping. \cite{multi2015li} applied the kernel trick to the label space. \cite{semi2015jing} added a trace norm regularization to identify the low-dimensional representation of the original space. To address the unsatisfactory accuracy caused by the violation of low rank assumption, \cite{bhatia2015sparse} learned a small ensemble of local distance preserving embeddings which non-linearly captured label correlations. \cite{rai2015large} presented a scalable Bayesian framework via a non-linear mapping. \cite{multi2016jian} decomposed the original label space to a low-dimensional space to reduce the noisy information in the label space. \cite{a2016wicker} proposed a model that compresses the labels by autoencoders and then used the same structure to decompress the labels, which is able to capture non-linear label dependencies. 

Previous researches on embedding mostly require an explicit encoding function for mapping the original labels to the embedding labels. However, since the optimal mapping can be complicated and even indescribable, assuming an explicit encoding function may not model it well. Unlike most previous works, some methods make no assumptions about the encoding process but directly learn a code matrix. \cite{multi2012wicker} proposed a method to perform LC via boolean matrix decomposition and \cite{multi2014lin} proposed a feature-aware implicit label space encoding method.

Quite different from the approach of existing LC methods, we propose CMLL following the spirit of CL. There are several quite related works have been proposed. From now on, we discuss the main differences between our proposal with them, and later in Section~\ref{sec:exp}, we experimentally validate our superiority. One is canonical correlated autoencoder (C2AE) proposed in \cite{learning2017yeh}, which performed joint feature and label embedding by deriving a deep latent space. The learned latent embedding space is shared by feature and label, thus C2AE restricted the embedded features and labels to have the same dimension. And the tasks of label embedding and multi-label prediction are integrated into the same framework. Another related work is co-hashing (CoH) proposed in \cite{shen2018compact}, which also learned a common latent hamming space to align the input and output for applying $k$-nearest neighbor ($k$NN) for predicting. 
Both of these two methods can be categorized as CL, but they are coupled with some specifical learning algorithms. While in our work, CMLL learns two respective subspaces for the features and the label. The classifier learned in the learning process is independent of the encoder and the decoder, so that any parametric or non-parametric learning model is compatible. In addition, the compression ratio of the two space in C2AE and CoH are mutual influenced and restricted. Different from them, CMLL produces respective compressed representations, which is more flexible on the compression ratio of each space.

\section{Proposed Algorithm}\label{sec:cmll}

The procedures of CMLL are quite similar to that of LC, bur CMLL needs to simultaneously learn another mapping for features, i.e., $e':\X\to\U$, where $\U=\RR^d$ is the embedded $d$-dimensional feature space ($d\le D$). For an unseen instance $\bm{x}_u$, the corresponding recovered relavant labels $\widehat{Y}_u$ outputed by CMLL are: $\widehat{Y}_u=d(g(e'(\bm{x}_u)))$, where $g:\U\to\V$ is a multi-label classifier.

In most cases, the classifier learned by a MLC system is a real-valued function \cite{a2014zhang} $f:\X\to\RR^c_+$, where $f(\bm x)$ can be regarded as the confidences of the labels being the relevant labels of $\bm x$. Specifically, given a multi-label example $(\bm x, Y)$, if $i\in Y, j\notin Y$, the $i$-element in $f(\bm x)$ should be larger than the $j$-element in $f(\bm x)$. Note that the multi-label classifier $g(\cdot)$ can be derived from the real-valued function $f(\cdot)$ via: $g(\bm x)=\II(f(\bm x),\delta)=\{i|f(\bm x)_i>\delta\}$, where $\delta\in (0,1)$ is a threshold and $f(\bm x)_i$ is the $i$-element of $f(\bm x)$.

\subsection{The Objective of CMLL}

As mentioned above, for boosting the performance, we should make the instances more \emph{predictable} in the learning process and the embedded label vectors more \emph{recoverable} in the decoding process. In this section, we propose a simple yet effective instantiation CMLL of CL framework.

It has been widely acknowledged that strong correlation usually leads to better predictability, hence CMLL maximizes the dependence between the embedded label space $\V$ and the embedded feature space $\U$. At the same time, CMLL minimizes the recovery loss from $\Y$ to $\V$. Let $\bm{X}\in\RR^{N\times D}$ be the feature matrix and $\bm{Y}\in\{0,1\}^{N\times M}$ be the corresponding label matrix. Given the training dataset $S=\{\bm{X},\bm{Y}\}$, we denote $\Omega(\bm{V},\bm{Y})$ as the recovery loss and $\Theta(\bm{V},\bm{U})$ as the measure of dependence, where $\bm{V}^{N\times m}$ is the embedded label matrix and $\bm{U}^{N\times d}$ is the the embedded feature matrix. Then, the objective can be formulized as follows.
\begin{equation}\label{eq:gamma}
\max \limits_{\bm{U},\bm{V}} ~ \alpha\Theta(\bm{V},\bm{U}) - \Omega(\bm{V},\bm{Y}),
\end{equation}
where $\alpha$ is a hyper-parameter that balances the importance of the dependence and the recovery loss. Next we discuss these two terms in Eq.~\ref{eq:gamma} respectively with the concrete form.

CMLL utilizes Hilbert-Schmidt Independence Criterion (HSIC) \cite{hsic05,multi2010zhang} as its dependence measurement due to its simple form and theoretical properties such as exponential convergence. HSIC calculates the squared norm of the cross-covariance operator over the domain $\X\times\Y$ in reproducing kernel Hilbert spaces. An empirical estimate of HSIC can be described as:
\begin{equation}
HSIC(\X, \Y) = (N-1)^{-2}~tr[\bm{HKHL}],
\end{equation}
where $tr[\cdot]$ denotes the trace of a matrix, $\bm{H} = \bm{I}-\frac{1}{N}\bm{ee}^{T}$, $~\bm{e}^{N\times 1}$ is an all-one vector, and $\bm{I}^{N\times N}$ is the unit matrix. $\bm{K}_{ij} = k(\bm{x}_{i}, \bm{x}_{j}) = \langle\phi(\bm{x}_{i}), \phi(\bm{x}_{j})\rangle$ and $\bm{L}_{ij} = l(\bm{y}_{i}, \bm{y}_{j}) = \langle\varphi(\bm{y}_{i}), \varphi(\bm{y}_{j})\rangle$, where $\langle\cdot\rangle$ represents the inner product operation, $k(\cdot)$ and $l(\cdot)$ are the kernel functions, and $\phi(\cdot)$ and $\varphi(\cdot)$ are the corresponding mapping functions. Dropping the normalization term of HSIC and applying it to CMLL, the measure of dependence can be represented as:
\begin{equation} \label{eq:depen0}
\Theta(\bm{V},\bm{U}) = tr[\bm{HKHL}],
\end{equation}
where $\bm{K}=\bm{U}\bm{U}^t$ and $\bm{L}=\bm{V}\bm{V}^t$. Here we first consider the linear embedding of features in CMLL, i.e. let $\bm{U}=\bm{XP}$, where $\bm{P}^{D\times d}$ is the learnt projection matrix. The kernel version of CMLL with non-linear feature embedding will also be derived later. Constraining the basis of the projection matrix to be orthonormal, we derive:
\begin{equation} \label{eq:depen}
\begin{aligned}
\Theta(\bm{V},\bm{U}) &= tr[\bm{HXP}\bm{P}^{t}\bm{X}^{t}\bm{H}\bm{V}\bm{V}^t] \\
&s.t. ~~ \bm{P}^t\bm{P}=\bm{I}.
\end{aligned}
\end{equation}

As to the second term of Eq.~(\ref{eq:gamma}), in order to minimize the loss of recovery, CMLL searches a decoding matrix $\bm{W}^{m\times M}$ through the ridge regression \cite{Hastie2004} to conduct a linear decoding. That is,
\begin{equation} \label{eq:omega}
\begin{aligned}
&\Omega(\bm{V},\bm{Y},\bm{W}) = \parallel \bm{Y}-\bm{\bm{V}W} \parallel^2_{F}+\lambda \parallel \bm{W} \parallel^2_F ,
\end{aligned}
\end{equation}
where $\parallel\cdot\parallel_{F}$ means the Frobenious norm and $\lambda$ is the coefficient of the regularization term that avoids overfitting. Given the specific $\bm{V}$ and $\bm{Y}$, the goal is to find the $\bm{W}$ to minimize $\Omega(\bm{V},\bm{Y},\bm{W})$. To avoid redundancy in the embedded label space, we assume that the components of the embedded space are orthonormal and uncorrelated, i.e.,
$\bm{V}^t\bm{V} = \bm{I}$. Then, let the partial derivative of $\Omega(\bm{V},\bm{Y},\bm{W})$ with respect to $\bm{W}$ be zero:
\begin{equation}
\begin{aligned}
\frac{\partial \Omega}{\partial \bm{W}}
&= \frac{tr[\bm{YY}^t + \bm{WW}^t \bm{V}^t\bm{V} - 2\bm{W}^t\bm{V}^t\bm{Y} + \lambda \bm{W}^t\bm{W} ]}{\partial \bm{W}} \\
&= 2\bm{V}^t\bm{V}\bm{W} - 2\bm{V}^t\bm{Y} + 2\lambda\bm{W} = 0.
\end{aligned}
\end{equation}
We can obtain:
\begin{equation}\label{eq:w}
\bm{W}=(\bm{V}^t\bm{V}+\lambda \bm{I})^{-1} \bm{V}^t \bm{Y} = \frac{1}{1+\lambda} \bm{V}^t\bm{Y}.
\end{equation}
Substituting Eq.~(\ref{eq:w}) into Eq.~(\ref{eq:omega}) and dropping unrelated items, we yield:
\begin{equation} \label{eq:recover}
\begin{aligned}
\Omega(\bm{V},\bm{Y}) &= -\frac{1}{1+\lambda} tr[ \bm{Y}^t \bm{V}\bm{V}^t \bm{Y}]
\\
& s.t. ~~ \bm{V}^t\bm{V}=\bm{I}.
\end{aligned}
\end{equation}

Then substituting Eq.~(\ref{eq:recover}) and Eq.~(\ref{eq:depen}) into Eq.~(\ref{eq:gamma}), the terms in the objective can be derived as follows.
\begin{equation}
\begin{aligned} \label{eq:derive}
& \alpha ~tr[\bm{H}\bm{U}\bm{U}^t\bm{H}\bm{V}\bm{V}^t] + \frac{1}{1+\lambda} tr[\bm{Y}^t \bm{V}\bm{V}^t \bm{Y}] \\
\Leftrightarrow &\alpha(1+\lambda) ~ tr[\bm{V}^t\bm{H}\bm{U}\bm{U}^t\bm{H}\bm{V}] +  tr[\bm{V}^t\bm{YY}^t\bm{V}] \\
\Leftrightarrow &
~ tr[\bm{V}^t (\beta \bm{H}\bm{U}\bm{U}^t\bm{H} + \bm{YY}^t)\bm{V}],
\end{aligned}
\end{equation}
where $\beta = {\alpha(1+\lambda)}$ is the normalized balance parameter. Adding the corresponding constraints, the learning objective becomes:
\begin{equation}
\begin{aligned} \label{eq:dual}
&\max\limits_{\bm{V},\bm{P}} ~ tr[\bm{V}^t ( \beta\bm{HXPP^t X^t H} + \bm{YY}^t)\bm{V}]  \\
&s.t. ~~~ \bm{V}^t\bm{V}=\bm{I}, ~\bm{P}^t\bm{P}=\bm{I}.
\end{aligned}
\end{equation}

\subsection{Solution for CMLL}\label{Solution of CMLL}

We solve Eq.~(\ref{eq:dual}) by alternating minimization. In each
iteration, fixing one of $\{\bm{P},\bm{V}\}$ and updating the other with coordinate ascent \cite{Hastie2004}, in which way a close-form solution can be obtained during each iteration.

To be specific, when $\bm{P}$ is fixed, the problem can be converted into an eigen-decomposition problem after applying the Lagrangian method. Let $\bm{A}=(\beta \bm{H}\bm{U}\bm{U}^t\bm{H}+\bm{YY}^t)$, the eigen-decomposition problem can be specified as:
\begin{equation}
\begin{aligned}  \label{eq:eigen}
& \max\limits_{\bm{V}}~ \sum_{j=1}^{m}\gamma_{j} \\
&s.t.~~~ \bm{A}\bm{V}_{.j} = \gamma_{j}\bm{V}_{.j}, ~\bm{V}_{.i}^t\bm{V}_{.j}= \II(i=j),
\end{aligned}
\end{equation}
where $\bm{V}_{.j}$ is the $j$-th column of $\bm{V}$, and $\gamma_{j}$ means the eigenvalue. The optimal $\bm{V}$ consists of $m$ normalized eigenvectors corresponding to the top $m$ largest eigenvalues of $\bm{A}$. Notice that $m$ is usually much smaller than $M$, so we can utilize some iterative approaches such as Arnoldi iteration \cite{iter96} to speed computation, which can reach a minimal computational complexity of $\mathcal{O}(Nm^2)$. When $\bm{V}$ is fixed, the optimal $\bm{P}$ consists of $d$ normalized eigenvectors corresponding to the top $d$ largest eigenvalues of $\bm{B} = \bm{X}^{t} \bm{H}\bm{V}\bm{V}^{t} \bm{HX}$.

The procedures of CMLL are summarized in Algorithm \ref{alg:CMLL}. It is interesting to note that if we replace $\bm{V}$ with $\bm{Y}$ in $\bm{B}$, regardless of the embedding for the labels, the solution for $\bm{P}$ is actually the same as MDDM \cite{multi2010zhang}, a typical FE method for MLC. And if we replace $\bm{U}$ with $\bm{X}$ in $\bm{A}$, a standard LC algorithm regardless of the embedding for the features can be derived, which we named as CMLL$_y$. Both MDDM and CMLL$_y$ can be viewed as the special case of CMLL.

\renewcommand{\algorithmicrequire}{\textbf{Input:}}
\renewcommand{\algorithmicensure}{\textbf{Output:}}
\begin{algorithm}[H]
	\caption{CMLL}
	\label{alg:CMLL}
	\begin{algorithmic}[1]
		\REQUIRE Training dataset $S=\{\bm{X},\bm{Y}\}$, testing feature matrix $\bm{X}_{test}$, parameter $\beta,\lambda$, dimensionality of the embedded label space $m$ and feature space $d$, maximal iteration count $maxc$, toleration $tol$.
		\ENSURE Predicted label matrix $\widehat{\bm{Y}}_{pre}$.
		
		\STATE Initialize $j=0,\bm{V}^0_{N*m}$, $\bm{P}^0_{D*d}$ with a random matrix.
		\STATE Get $\Gamma^0 = tr[\bm{V}^0 (\beta \bm{HXP}^0(\bm{P}^0)^{t}\bm{X}^{t}\bm{H} + \bm{YY}^t )\bm{V}^0]$.
		
		\REPEAT
		\STATE Get $\bm{A}^{j+1} = \beta \bm{HXP}^{j+1}(\bm{P}^{j+1})^t\bm{X}^t\bm{H}+\bm{YY}^t$, then obtain $\bm{V}^{j+1}$ via eigen-decomposition.
		\STATE Get $\bm{B}^{j+1} = \bm{X}^{t} \bm{H}\bm{V}^{j+1}(\bm{V}^{j+1})^{t} \bm{HX}$, then obtain $\bm{P}^{j+1}$ via eigen-decomposition.
		\STATE Get $\Gamma^{j+1}$ using $\bm{P}^{j+1}$ and  $\bm{V}^{j+1}$
		\STATE Compute $\Delta= |\Gamma^{j+1} - \Gamma^{j}|~/~(\Gamma^{j})$.
		\STATE Let $j=j+1$, $\bm{P}=\bm{P}^{j}, \bm{V}=\bm{V}^{j}$.
		\UNTIL ($j>maxc$) ~or~ ($\Delta<tol$)
		
		\STATE Compute $\bm{W}=\frac{1}{1+\lambda} \bm{V}^t\bm{Y}$.
		\STATE Learn the classifier: $g:\bm{XP}\to\bm{V}$.
		\STATE Conduct prediction: $\bm{V}_{pre}=g(\bm{X}_{test}\bm{P})$.
		\STATE Perform decoding: $\widehat{\bm{Y}}_{pre}=\II(\bm{V}_{pre}\bm{W},\delta)$.
	\end{algorithmic}
\end{algorithm}

\subsection{Kernelization for CMLL}
\label{sec:kerenl}
We can utilize kernel tricks to extend CMLL to the non-linear case, denoted by k-CMLL. Assume the projection matrix $\bm{P}$ can be spanned by kernel feature vectors, i.e. $ \bm{P} = \Phi(\bm{X}) \bm{R}^{N\times d} $, where $\Phi(\bm{X}) = [\phi(\bm{x}_{1}),\phi(\bm{x}_{2}),..., \phi(\bm{x}_{N} )]$, and $\phi(.)$ is the projection function corresponding to the kernel and $\bm{R}$ is the matrix of the corresponding linear combination coefficients.

Let $q(\bm{x}_{i}, \bm{x}_{j})=\langle\phi(\bm{x}_{i}), \phi(\bm{x}_{j})\rangle$ be the chosen kernel function and $\bm{Q}=\Phi(\bm{X})^{t} \Phi(\bm{X})$ be the kernel matrix.
Then, $\bm{U}=\Phi(\bm{X})^{t} \bm{P} = \bm{Q} \bm{R}, ~\bm{K}=\bm{U}\bm{U}^t = \bm{QRR}^{t}\bm{Q}$, and the constraint $\bm{P}^t\bm{P} = \bm{R}^{t} \bm{Q} \bm{R} = \bm{I}$.
So the objective of the kernel CMLL becomes:
\begin{equation}
\begin{aligned} \label{eq:kernel}
&\max\limits_{\bm{O,R}} ~  tr[\bm{V}^t ( \beta \bm{H}\bm{QRR}^{t}\bm{Q}\bm{H} + \bm{YY}^t )\bm{V}], \\
&s.t. ~~~\bm{V}^t\bm{V}=\bm{I}, ~\bm{R}^{t} \bm{Q}\bm{R} = \bm{I}.
\end{aligned}
\end{equation}

The solution for k-CMLL is similar to that of the linear case. When $\bm{R}$ is fixed, the optimal $\bm{V}$ consists of the top $m$ eigenvectors of $\bm{A'}=\beta \bm{H}\bm{QRR}^{t}\bm{Q}\bm{H} + \bm{YY}^t$. And when $\bm{V}$ is fixed, the optimal $\bm{R}$ consists of the top $d$ generalized eigenvectors of $\bm{B'}=\bm{QH}\bm{V}\bm{V}^{t} \bm{HQ}$ and $\bm{Q}$. Given an unseen instance $\bm{x}$, the projection is $\bm{z} = \bm{P}^{t} \phi(\bm{x})  = \bm{R}^{t} q(\bm{X},\bm{x})$, where $q(\bm{X},\bm{x}) = [q(\bm{x_}{1}, \bm{x}), q(\bm{x}_{2}, \bm{x}), \cdots, q(\bm{x}_{N}, \bm{x}) ]^{t}$.

\section{Theoretical Analysis}\label{sec:theory}
This section will conduct a general theoretical analysis for different embedding strategies and compare them basing on the proposed Theorem \ref{thm1}.

\newtheorem{thm}{Theorem}
\begin{thm}
	\label{thm1}
	Given an instance $\bm{x}$ in a multi-label dataset $S$, denote $\bm{y}$ as its true label vector, and $\bm{\widehat{y}}$ as its predicted real-valued label vector obtained via a specific embedding framework. Assume a fixed threshold $\delta \in (0,1)$ is used in the final step to binarize $\bm{\widehat{y}}$. Denote $n_{mis}$ as the number of misclassified labels, then $n_{mis}$ is upper-bounded by:
	\begin{equation*}
	n_{mis} \le ~ \tau {\parallel \widehat{\bm{y}} - \bm{y} \parallel}^2,
	~~~~\tau = ~max ~ \{ \frac{1}{\delta^2},\frac{1}{(1-\delta)^2} \} .
	\end{equation*}
\end{thm}

\begin{proof}
	Denote $\widehat{\bm{y}}^i$ as the $i$-th dimension of $\widehat{\bm{y}}$, and $\bm{y}^i$ as the $i$-th dimension of $\bm{y}$, then
	\begin{align}\label{eq:thm0}
	& n_{mis} = \sum_{i=1}^{M} \II( \widehat{\bm{y}}^i \ge \delta ) \II( \bm{y}^i = 0 ) + \II( \widehat{\bm{y}}^i < \delta ) \II( \bm{y}^i = 1 ) \\\nonumber
	&\le  \sum_{i=1}^{M} \frac{{( \widehat{\bm{y}}^i - \bm{y}^i )}^2 \II( \bm{y}^i = 0 )}{\delta^2}  + \frac{{( \widehat{\bm{y}}^i - \bm{y}^i )}^2 \II( \bm{y}^i = 1 )}{(1-\delta)^2} \\\nonumber
	&\le \sum_{i=1}^{M}  max\{ \frac{1}{\delta^2},\frac{1}{(1-\delta)^2} \} ~~ {( \widehat{\bm{y}}^i - \bm{y}^i )}^2 \\\nonumber
	& = \tau {\parallel \widehat{\bm{y}} - \bm{y} \parallel}^2. 
	\end{align}
\end{proof}

Notice that when $\delta = 0.5$, the condition of Eq.~(\ref{eq:thm0}) inequality taking mark of equality is satisfied so that $\tau$ reaches the minimum value $4$. For MLC problem, it is a common practice that FE, LC and CL all binarize the real-valued output by $\II(\cdot,\delta)$. From now on, we can make an preliminary analysis on the error upper bounds of these different embedding frameworks based on Theorem~\ref{thm1}. Assume that $h:\U\to\Y$, a natural result inferred directly from the Theorem \ref{thm1} is that:
\begin{equation*}
\begin{aligned}
&Z_{FE} = ~\tau~ {\parallel h(e'(\bm{x})) - \bm{y} \parallel}^2,   \\
&Z_{LC} = ~\tau~ {\parallel d(g'(\bm{x})) - \bm{y} \parallel}^2,   \\
&Z_{CL} = ~\tau~ {\parallel d(g(e'({\bm{x}}))) - \bm{y}] \parallel}^2,
\end{aligned}
\end{equation*}
where $Z_{FE},Z_{LC},Z_{CL}$ are the error upper bound, i.e., the number of misclassified labels for the instance $\bm x$ by FE, LC, CL, respectively. The common purpose is to minimize the error bound by a embedding framework.

While in practical implementation, LC usually formalizes $Z_{LC}$ as an equivalent form:
\begin{equation}
\label{eq:ulsdr2}
Z_{LC}' = \tau {\parallel [ d(g'(\bm{x})) - d(e(\bm{y})) ] + [d(e(\bm{y})) - \bm{y}] \parallel}^2,
\end{equation}
because directly minimizing the original $Z_{LC}$ is not intuitional and feasible when designing a concrete algorithm \cite{multi2012tai,feature2012chen,dependence2015zhang,multi2014lin,multi2015li}. That is, LC tries to make the encoded labels $e(\bm{y})$ more predictable for $\bm{x}$ and more recoverable to $\bm{y}$. We can specify this general error bound for a specific method. For example, substituting the concrete form of PLST to Eq.~(\ref{eq:ulsdr2}) yields:
\begin{equation}
Z_{PLST}={4 \parallel [g'(\bm{x}) - \bm{yO}]\bm{O}^{t} + [ \bm{yOO}^{t} - \bm{y} ] \parallel}^2,
\end{equation}
where $\bm{O}^{M\times m}$ is the orthonormal projection matrix learnt by PLST. Note that $Z_{PLST}$ derived above can be transformed into the form derived in \cite{multi2012tai} through some equivalent conversion.

Similarly, $Z_{CL}$ can be formalized as follows:
\begin{equation}
\label{eq:ucmll2}
Z'_{CL} = \tau {\parallel [d(g(e'({\bm{x}}))) - d(e(\bm{y})) ] + [d(e(\bm{y})) - \bm{y}] \parallel}^2.
\end{equation}
Compared to FE and LC, CL considers the transformation for both label and feature space simultaneously, and thus provides a greater possibility as well as a more flexible and superior way to make upper bound tighter. Think of a specific example CMLL, the error bound can be expressed as:
\begin{equation}
{4 \parallel [g(\bm{u}) - \bm{v}^tW] + [ \bm{v}^tW - \bm{y} ] \parallel}^2.
\end{equation}
The derivation process of CMLL shows that CMLL indeed takes both terms into account, i.e., minimizes the first term by maximizing the dependence between the two embedded spaces, and at the same time minimizes the recovery loss that measures how well $\bm{v}^tW$ approximates $\bm y$. Besides, CL can also degenerate to FE or LC when necessary, as the example of special cases of CMLL (i.e. MDDM and CMLL$_{\bm{y}}$) indicates.

The upper bound derived here seems loose because it aims at embedding strategies rather than any concrete algorithm. There are few research on the analysis of the framework of embedding yet, although many related methods have been proposed. This section makes an initial attempt to analyze the reasonability of CL as well as LC and FE, on which existing LC methods can be explained/derived based. It provides guidance on the aspects that should be considered when designing a new CL or LC algorithm.

\section{Experiments}\label{sec:exp}
\subsection{Datasets}
Aiming to validate the effectiveness of CMLL, we conduct experiments on a total of twelve public real-world multi-label datasets, which show obvious label sparsity. The dataset of espgame collected here is organized by \cite{multi2014lin}, and other small-scale datasets (the number of examples is less than 5000) can be downloaded from Mulan \footnote{http://mulan.sourceforge.net/datasets-mlc.html} and Meka \footnote{http://meka.sourceforge.net/\#datasets}. In addition, the extreme multi-label learning \cite{wei2019learning} which aims to learn relevant labels from an extremely large label set is a possible application domain of LC, thus, we also adopt two extreme classification datasets Mediamill and Delicious \footnote{http://manikvarma.org/downloads/XC/XMLRepository.html}.
Table~\ref{tb:datasets} summarizes the detailed characteristics of these datasets, which are organized in ascending order of the number of examples. Cardinality means the average number of relevant labels per instance, Density is the ratio of Cardinality to the number of classes, and Distinct is the number of distinct label combinations contained in the dataset. As indicated by quite small values of Density and Distinct compared to all the possible label combinations (i.e. $2^{\#Label}$), all datasets suffer from evident hypercube sparsity or label-set sparsity in the label space.

\begin{table*}[!t]
	\begin{center}
		\caption{Characteristics of the real-world multi-label datasets.}  \label{tb:datasets}
		\vskip 0.15in
		\begin{small}
		\setlength{\tabcolsep}{2mm}{
			\begin{tabular}{ c|cccccccc}
				\toprule
				Datasets & \#Label & \#Feature & \#Example & Feature Type & Cardinality & Density & Distinct & Domain \\
				\midrule
				plant & 12 & 440 & 978 & numeric &  1.079 & 0.090 & 32 &  biology\\ \hline
				msra & 19 & 898 & 1,868 & numeric & 6.315 & 0.332 & 947 & images \\ \hline
				enron & 53 & 1,001 & 1,702 & nominal & 3.378 & 0.064 & 753 & text \\ \hline
				llog & 74 & 1,004 & 1,460 & nominal & 1.128 & 0.015 & 286 & text \\ \hline
				
				bibtex & 159 & 1,836 & 5,000 & nominal & 2.397 & 0.015 & 2,127 & text\\ \hline
				eurlex-sm & 201 & 5,000 & 5,000 & numeric & 2.224 & 0.011 & 1,236 & text \\ \hline
				bookmarks & 208 & 2,150 & 5,000 & nominal & 2.016 & 0.010 & 1,840 & text \\ \hline
				corel5k & 374 & 499 & 5,000 & nominal & 3.522 & 0.009 & 3,175 &  images\\ \hline
				eurlex-dc & 412 & 5,000 & 5,000 & numeric & 1.296 & 0.003 & 859 & text \\ \hline
				espgame & 1,932 & 516 & 5,000 & numeric & 4.689 & 0.002 & 4,734 & images \\ \hline
				Delicious & 983 & 500 & 16,105 & numeric & 19.020 & 0.002 & 15,806 & text \\ \hline
				Mediamill & 101 & 120 & 43,907 & numeric & 4.376 & 0.004 & 6,555 & vedio \\ 
				\bottomrule
		\end{tabular}}
	\end{small}
	\end{center}
	\vskip -0.1in
\end{table*}

\subsection{Setups}

We compare CMLL with its special cases CMLL$_{\bm{y}}$, one FE algorithm MDDM \cite{multi2010zhang}, one state-of-the-art large-scale multi-label learning algorithm POP \cite{wei2019learning}, and five well-established LC algorithms: PLST \cite{multi2012tai}, CPLST \cite{feature2012chen}, FaIE \cite{multi2014lin}, DMLR \cite{dependence2015zhang} and C2AE \cite{learning2017yeh}.

The hyper-parameters of the baselines were selected according to the suggested parameter settings in original papers. The balance parameter of CMLL, CMLL$_{\bm{y}}$, FaIE, DMLR was selected from $\{10^{-5}, 10^{-4}, ..., 10^{4}, 10^{5}\}$. POP used Binary Relevance as the base classifier. The hyper-parameter $\lambda$ of CMLL and CMLL$_{\bm{y}}$ was selected from $\{0, 10^{-3}, 10^{-1}\}$, and $tol = 10^{-5}$, $maxc = 50$. And $\delta=0.5$ in the final step for binarizing the real-valued outputs. Besides, we also compared our kernel versions k-CMLL and k-CMLL$_{\bm{y}}$ with the baselines (except PLST and POP that can hardly be extended to the kernel version) and C2AE, which adopted the DNN architectures. And the RBF kernel was applied. Following the previous works \cite{multi2012tai,dependence2015zhang}, we used the ridge regression and the kernel ridge regression to train the learning model for linear case and kernel case, respectively. We denote ORI to represent the classifier learning from the original spaces as the baseline. The regulation parameter of the ridge regression is selected from $\{10^{-5}, 10^{-4}, ..., 10^{-1}\}$. 

Denoting $\mu = \frac{d}{D}, \nu = \frac{m}{M}$ the feature and the label compression ratio, all LC methods run with $\mu$ ranging from $10\%$ to $100\%$ with the interval of $10\%$ while MDDM runs with $\nu$ similarly. CMLL and C2AE run with both $\nu$ and $\mu$. That means, CMLL needs to run with 100 ratio pairs ($10\times10$) in total while C2AE run with 20 ratio pairs ($10+10$). Because C2AE essentially conducts non-linear embedding by utilizing the DNN structure and learns a shared embedded space for both labels and features, while CMLL learns two sub-spaces for labels and features respectively. 

To measure the performance, we use seven widely adopted metrics in multi-label classification, including \emph{Average Precision}, \emph{micro-F1}, \emph{Ranking Loss} and \emph{One Error}. The concrete definition of these metrics can be found in \cite{a2014zhang}. For Mediamill and Delicious, we supplement two metrics popularly used in extreme multi-label learning: \emph{Precision$@3$} and \emph{nDCG$@3$}.

\begin{table*} [tp]
	\huge
	\begin{center}
		\caption{Experimental results of CMLL with baselines.} \label{tb:linear}
		\begin{minipage}{1\linewidth}
			\centering
			\vskip 0.15in
			\resizebox{\textwidth}{9cm}{
				\begin{tabular}{c|ccccccccc}
					\toprule
					
					& \multicolumn{9}{c}{Average Precision $\uparrow$}  \\ \midrule
					Methods & ORI & PLST & CPLST & DMLR & FaIE & CMLL$_{\bm{y}}$ & MDDM & POP & CMLL \\ \hline
					plant             & 0.4756$\pm$0.0263$\bullet$ & 0.4873$\pm$0.0299$\bullet$ & 0.4848$\pm$0.0274$\bullet$ & 0.4859$\pm$0.0250$\bullet$ & 0.4813$\pm$0.0282$\bullet$ & 0.5399$\pm$0.0311 & 0.5587$\pm$0.0294 & 0.5256$\pm$0.0707 & \bf{0.5733$\pm$0.0252} \\
					msra              & 0.7433$\pm$0.0129$\bullet$ & 0.7733$\pm$0.0107$\bullet$ & 0.7689$\pm$0.0117$\bullet$ & 0.7766$\pm$0.0129$\bullet$ & 0.7782$\pm$0.0113$\bullet$ & 0.7820$\pm$0.0112$\bullet$ & 0.7933$\pm$0.0090 & 0.7315$\pm$0.0134$\bullet$ & \bf{0.7997$\pm$0.0100} \\
					enron             & 0.5039$\pm$0.0111$\bullet$ & 0.5183$\pm$0.0162$\bullet$ & 0.5170$\pm$0.0160$\bullet$ & 0.5196$\pm$0.0110$\bullet$ & 0.5218$\pm$0.0123$\bullet$ & 0.6228$\pm$0.0119$\bullet$ & 0.6767$\pm$0.0205 & 0.6360$\pm$0.0867 & \bf{0.6965$\pm$0.0153} \\
					llog              & 0.2235$\pm$0.0274$\bullet$ & 0.2429$\pm$0.0314$\bullet$ & 0.2412$\pm$0.0320$\bullet$ & 0.2760$\pm$0.0328$\bullet$ & 0.2722$\pm$0.0344$\bullet$ & 0.3413$\pm$0.0378$\bullet$ & \bf{0.4264$\pm$0.0216} & 0.1532$\pm$0.0563$\bullet$ & 0.3717$\pm$0.0283 \\
					bibtext           & 0.5278$\pm$0.0091$\bullet$ & 0.5258$\pm$0.0077$\bullet$ & 0.5280$\pm$0.0079$\bullet$ & 0.5281$\pm$0.0086$\bullet$ & 0.5259$\pm$0.0078$\bullet$ & 0.5241$\pm$0.0039$\bullet$ & 0.5287$\pm$0.0091$\bullet$ & 0.5676$\pm$0.0069 & \bf{0.5828$\pm$0.0050} \\
					eurlex-sm         & 0.3972$\pm$0.0208$\bullet$ & 0.3998$\pm$0.0209$\bullet$ & 0.3998$\pm$0.0209$\bullet$ & 0.3979$\pm$0.0208$\bullet$ & 0.4289$\pm$0.0219$\bullet$ & 0.5706$\pm$0.0119$\bullet$ & 0.7299$\pm$0.0198$\bullet$ & 0.3419$\pm$0.0058$\bullet$ & \bf{0.7565$\pm$0.0084} \\
					bookmark          & 0.3086$\pm$0.0059$\bullet$ & 0.3043$\pm$0.0065$\bullet$ & 0.3030$\pm$0.0059$\bullet$ & 0.3054$\pm$0.0064$\bullet$ & 0.3051$\pm$0.0061$\bullet$ & 0.3161$\pm$0.0064$\bullet$ & 0.3804$\pm$0.0060$\bullet$ & 0.4024$\pm$0.0088 & \bf{0.4080$\pm$0.0078} \\
					corel5k           & 0.2892$\pm$0.0029$\bullet$ & 0.2900$\pm$0.0035$\bullet$ & 0.2908$\pm$0.0039$\bullet$ & 0.2932$\pm$0.0030$\bullet$ & 0.2916$\pm$0.0035$\bullet$ & 0.2925$\pm$0.0035$\bullet$ & 0.2995$\pm$0.0057 & 0.0900$\pm$0.0076$\bullet$ & \bf{0.3028$\pm$0.0070} \\
					eurlex-dc          & 0.3982$\pm$0.0315$\bullet$ & 0.4031$\pm$0.0304$\bullet$ & 0.4031$\pm$0.0304$\bullet$ & 0.4511$\pm$0.0240$\bullet$ & 0.5031$\pm$0.0305$\bullet$ & 0.6118$\pm$0.0254$\bullet$ & 0.6911$\pm$0.0215$\bullet$ & 0.4616$\pm$0.0089$\bullet$ & \bf{0.7588$\pm$0.0144} \\
					espgame           & 0.2171$\pm$0.0081 & 0.2171$\pm$0.0081 & 0.2173$\pm$0.0081 & 0.2175$\pm$0.0082 & 0.2169$\pm$0.0081 & 0.2169$\pm$0.0081 & 0.2175$\pm$0.0083 & 0.0141$\pm$0.0006$\bullet$ & \bf{0.2177$\pm$0.0080} \\
					Delicious        & 0.3338$\pm$0.0029$\bullet$ & 0.3443$\pm$0.0024$\bullet$ & 0.3450$\pm$0.0025$\bullet$ & 0.3337$\pm$0.0030$\bullet$ & 0.3337$\pm$0.0029$\bullet$ & 0.3485$\pm$0.0028$\bullet$ & 0.3514$\pm$0.0025 & 0.2101$\pm$0.0039$\bullet$ & \bf{0.3543$\pm$0.0030} \\ 
					Mediamill        & 0.7193$\pm$0.0031$\bullet$ & 0.7217$\pm$0.0033$\bullet$ & 0.7200$\pm$0.0031$\bullet$ & 0.7218$\pm$0.0035$\bullet$ & 0.7195$\pm$0.0034$\bullet$ & 0.7171$\pm$0.0035$\bullet$ & 0.7193$\pm$0.0031$\bullet$ & 0.5389$\pm$0.0096$\bullet$ & \bf{0.7302$\pm$0.0031} \\ \midrule\midrule
					
					& \multicolumn{9}{c}{micro-F1 $\uparrow$}  \\ \midrule
					Methods   & ORI            & PLST           & CPLST          & DMLR           & FaIE           & CMLL$_{\bm{y}}$          & MDDM           & POP             & CMLL           \\\hline
					plant     & 0.2644$\pm$0.0312 & 0.2644$\pm$0.0312 & 0.2646$\pm$0.0288 & 0.2460$\pm$0.0256$\bullet$ & 0.2663$\pm$0.0319 & 0.2991$\pm$0.0458 & 0.2652$\pm$0.0322 & 0.2222$\pm$0.0247$\bullet$ & \bf{0.2993$\pm$0.0190} \\
					msra      & 0.6471$\pm$0.0134$\bullet$ & 0.6644$\pm$0.0076$\bullet$ & 0.6606$\pm$0.0126$\bullet$ & 0.6605$\pm$0.0125$\bullet$ & 0.6641$\pm$0.0118$\bullet$ & 0.6677$\pm$0.0118 & 0.6757$\pm$0.0081 & 0.5942$\pm$0.0135$\bullet$ & \bf{0.6817$\pm$0.0096} \\
					enron     & 0.4045$\pm$0.0057$\bullet$ & 0.4480$\pm$0.0105$\bullet$ & 0.4473$\pm$0.0138$\bullet$ & 0.4437$\pm$0.0121$\bullet$ & 0.4579$\pm$0.0109$\bullet$ & 0.4995$\pm$0.0119$\bullet$ & 0.5278$\pm$0.0057 & 0.4913$\pm$0.0410 & \bf{0.5335$\pm$0.0147} \\
					llog      & 0.1390$\pm$0.0094$\bullet$ & 0.1710$\pm$0.0177$\bullet$ & 0.1726$\pm$0.0178$\bullet$ & 0.1754$\pm$0.0180$\bullet$ & 0.1724$\pm$0.0196$\bullet$ & 0.1926$\pm$0.0223$\bullet$ & 0.2289$\pm$0.0114 & 0.0942$\pm$0.0292$\bullet$ & \bf{0.2483$\pm$0.0175} \\
					bibtext   & \bf{0.3942$\pm$0.0107}$\circ$ & 0.3910$\pm$0.0101$\circ$ & 0.3939$\pm$0.0117$\circ$ & 0.3921$\pm$0.0106$\circ$ & 0.3914$\pm$0.0099$\circ$ & 0.3764$\pm$0.0072$\bullet$ & 0.3940$\pm$0.0102$\circ$ & 0.3566$\pm$0.0039$\bullet$ & 0.3683$\pm$0.0104 \\
					eurlex-sm & 0.1181$\pm$0.0073$\bullet$ & 0.1225$\pm$0.0078$\bullet$ & 0.1225$\pm$0.0078$\bullet$ & 0.1204$\pm$0.0072$\bullet$ & 0.2220$\pm$0.0080$\bullet$ & 0.2327$\pm$0.0094$\bullet$ & 0.3235$\pm$0.0079 & \bf{0.3393$\pm$0.0152} & 0.3350$\pm$0.0125 \\
					bookmark  & 0.1616$\pm$0.0094$\bullet$ & 0.1879$\pm$0.0080$\bullet$ & 0.1882$\pm$0.0073$\bullet$ & 0.1927$\pm$0.0091$\bullet$ & 0.1879$\pm$0.0081$\bullet$ & 0.2218$\pm$0.0100$\bullet$ & 0.2163$\pm$0.0085$\bullet$ & 0.2176$\pm$0.0060$\bullet$  & \bf{0.2386$\pm$0.0092} \\
					corel5k   & 0.1032$\pm$0.0050$\bullet$ & 0.1002$\pm$0.0064$\bullet$ & 0.0998$\pm$0.0070$\bullet$ & 0.1017$\pm$0.0047$\bullet$ & \bf{0.1523$\pm$0.0051}$\circ$ & 0.1258$\pm$0.0044 & 0.1201$\pm$0.0046 & 0.1456$\pm$0.0084$\circ$  & 0.1266$\pm$0.0048 \\
					eurlex-dc  & 0.0527$\pm$0.0035$\bullet$ & 0.0545$\pm$0.0035$\bullet$ & 0.0545$\pm$0.0035$\bullet$ & 0.0889$\pm$0.0057$\bullet$ & 0.1745$\pm$0.0039$\bullet$ & 0.2486$\pm$0.0107$\bullet$ & 0.2588$\pm$0.0034$\bullet$ & 0.2251$\pm$0.0015$\bullet$  & \bf{0.3735$\pm$0.0183} \\
					espgame   & 0.0863$\pm$0.0054 & 0.0863$\pm$0.0055 & 0.0858$\pm$0.0048 & 0.0859$\pm$0.0050 & 0.0861$\pm$0.0052 & \bf{0.0863$\pm$0.0051} & 0.0833$\pm$0.0036 & 0.0098$\pm$0.0003$\bullet$ & 0.0860$\pm$0.0055 \\
					Delicious  & 0.1614$\pm$0.0033$\bullet$ & 0.1639$\pm$0.0041$\bullet$ & 0.1620$\pm$0.0041$\bullet$ & 0.1615$\pm$0.0033$\bullet$ & 0.1608$\pm$0.0032$\bullet$ & 0.1635$\pm$0.0033$\bullet$ & 0.1607$\pm$0.0030$\bullet$ & \bf{0.2120$\pm$0.0031}$\circ$  & 0.2002$\pm$0.0032  \\
					Mediamill  & 0.5315$\pm$0.0021$\bullet$ & 0.5351$\pm$0.0022$\bullet$ & 0.5353$\pm$0.0021$\bullet$ & 0.5354$\pm$0.0022$\bullet$ & 0.5356$\pm$0.0022$\bullet$ & 0.5367$\pm$0.0023$\bullet$ & 0.5315$\pm$0.0020$\bullet$ & 0.4682$\pm$0.0043$\bullet$  & \bf{0.5427$\pm$0.0019}  \\ \midrule\midrule
					
					& \multicolumn{9}{c}{Ranking Loss $\downarrow$}  \\ \midrule
					Methods & ORI & PLST & CPLST & DMLR & FaIE & CMLL$_{\bm{y}}$ & MDDM & POP & CMLL \\ \hline
					plant        & 0.3292$\pm$0.0166$\bullet$ & 0.3109$\pm$0.0193$\bullet$ & 0.3261$\pm$0.0174$\bullet$ & 0.3276$\pm$0.0239$\bullet$ & 0.3250$\pm$0.0189$\bullet$ & 0.2594$\pm$0.0221$\bullet$ & 0.2123$\pm$0.0240 & 0.2441$\pm$0.0567 & \bf{0.2099$\pm$0.0156} \\
					msra         & 0.1938$\pm$0.0103$\bullet$ & 0.1677$\pm$0.0105$\bullet$ & 0.1742$\pm$0.0118$\bullet$ & 0.1663$\pm$0.0116$\bullet$ & 0.1642$\pm$0.0109$\bullet$ & 0.1611$\pm$0.0105$\bullet$ & 0.1531$\pm$0.0061 & 0.2208$\pm$0.0117$\bullet$ & \bf{0.1463$\pm$0.0081} \\
					enron        & 0.2672$\pm$0.0108$\bullet$ & 0.2596$\pm$0.0102$\bullet$ & 0.2592$\pm$0.0110$\bullet$ & 0.2615$\pm$0.0119$\bullet$ & 0.2642$\pm$0.0071$\bullet$ & 0.1475$\pm$0.0104$\bullet$ & 0.1382$\pm$0.0064$\bullet$ & 0.1255$\pm$0.0145 & \bf{0.1209$\pm$0.0093} \\
					llog         & 0.2567$\pm$0.0306$\bullet$ & 0.2610$\pm$0.0276$\bullet$ & 0.2637$\pm$0.0306$\bullet$ & 0.2585$\pm$0.0271$\bullet$ & 0.2625$\pm$0.0295$\bullet$ & 0.1538$\pm$0.0283$\bullet$ & 0.1661$\pm$0.0151$\bullet$ & 0.1345$\pm$0.0385 & \bf{0.1163$\pm$0.0158} \\
					bibtext      & 0.1325$\pm$0.0082$\bullet$ & 0.1319$\pm$0.0080$\bullet$ & 0.1319$\pm$0.0073$\bullet$ & 0.1318$\pm$0.0079$\bullet$ & 0.1319$\pm$0.0080$\bullet$ & 0.0873$\pm$0.0069 & 0.1020$\pm$0.0081$\bullet$ & \bf{0.0421$\pm$0.0071}$\circ$ & 0.0834$\pm$0.0062 \\
					eurlex-sm    & 0.2396$\pm$0.0107$\bullet$ & 0.2386$\pm$0.0088$\bullet$ & 0.2386$\pm$0.0087$\bullet$ & 0.2397$\pm$0.0103$\bullet$ & 0.1504$\pm$0.0102$\bullet$ & 0.0887$\pm$0.0083$\bullet$ & 0.0524$\pm$0.0119 & 0.0979$\pm$0.0125$\bullet$ & \bf{0.0481$\pm$0.0022} \\
					bookmark     & 0.2563$\pm$0.0072$\bullet$ & 0.2577$\pm$0.0066$\bullet$ & 0.2581$\pm$0.0060$\bullet$ & 0.2571$\pm$0.0057$\bullet$ & 0.2602$\pm$0.0071$\bullet$ & 0.1710$\pm$0.0068 & 0.2044$\pm$0.0071$\bullet$ & \bf{0.1493$\pm$0.0042} & 0.1642$\pm$0.0032 \\
					corel5k      & 0.2096$\pm$0.0044$\bullet$ & 0.1937$\pm$0.0047 & 0.1957$\pm$0.0053 & 0.1943$\pm$0.0040 & 0.1988$\pm$0.0035$\bullet$ & 0.1944$\pm$0.0045 & 0.1964$\pm$0.0070 & 0.4630$\pm$0.0079$\bullet$ & \bf{0.1880$\pm$0.0059} \\
					eurlex-dc     & 0.1841$\pm$0.0110$\bullet$ & 0.1838$\pm$0.0095$\bullet$ & 0.1839$\pm$0.0096$\bullet$ & 0.1943$\pm$0.0126$\bullet$ & 0.0941$\pm$0.0111$\bullet$ & 0.0442$\pm$0.0112 & 0.0477$\pm$0.0089 & 0.1486$\pm$0.0053$\bullet$ & \bf{0.0390$\pm$0.0026} \\
					espgame      & 0.2439$\pm$0.0024$\bullet$ & 0.2436$\pm$0.0026$\bullet$ & 0.2386$\pm$0.0035$\bullet$ & 0.2422$\pm$0.0025$\bullet$ & 0.2450$\pm$0.0030$\bullet$ & 0.2445$\pm$0.0027$\bullet$ & \bf{0.1926$\pm$0.0032} & 0.2858$\pm$0.0035$\bullet$ & 0.1932$\pm$0.0025 \\ 
					Delicious      & 0.1755$\pm$0.0026 & 0.1710$\pm$0.0010$\bullet$ & 0.1654$\pm$0.0008 & 0.1656$\pm$0.0026 & 0.1634$\pm$0.0024 & 0.1681$\pm$0.0023 & 0.1652$\pm$0.0025 & 0.3608$\pm$0.0043$\bullet$ & \bf{0.1651$\pm$0.0024}  \\ 
					Mediamill      & 0.0587$\pm$0.0008 & 0.0589$\pm$0.0007$\bullet$ & 0.0599$\pm$0.0006$\bullet$ & 0.0590$\pm$0.0008$\bullet$ & 0.0598$\pm$0.0010$\bullet$ & 0.0585$\pm$0.0009 & 0.0587$\pm$0.0008 & 0.2272$\pm$0.0120$\bullet$ & \bf{0.0576$\pm$0.0009}  \\ \midrule\midrule
					
					& \multicolumn{9}{c}{One Error $\downarrow$}  \\ \midrule
					Methods & ORI & PLST & CPLST & DMLR & FaIE & CMLL$_{\bm{y}}$ & MDDM & POP & CMLL \\ \hline
					plant     & 0.7099$\pm$0.0337$\bullet$ & 0.7089$\pm$0.0350$\bullet$ & 0.7058$\pm$0.0374$\bullet$ & 0.6986$\pm$0.0335$\bullet$ & 0.7058$\pm$0.0365$\bullet$ & 0.6547$\pm$0.0393 & 0.6362$\pm$0.0389 & 0.6649$\pm$0.0820 & \bf{0.6270$\pm$0.0252} \\
					msra      & 0.1304$\pm$0.0166$\bullet$ & 0.1015$\pm$0.0205$\bullet$ & 0.0972$\pm$0.0097$\bullet$ & 0.0791$\pm$0.0166$\bullet$ & 0.0844$\pm$0.0167$\bullet$ & 0.0796$\pm$0.0174 & 0.0646$\pm$0.0160 & \bf{0.0507$\pm$0.0129} & 0.0576$\pm$0.0142 \\
					enron     & 0.4454$\pm$0.0234$\bullet$ & 0.4471$\pm$0.0209$\bullet$ & 0.4483$\pm$0.0223$\bullet$ & 0.4318$\pm$0.0187$\bullet$ & 0.4295$\pm$0.0221$\bullet$ & 0.2930$\pm$0.0236$\bullet$ & 0.2524$\pm$0.0308 & 0.3535$\pm$0.2583 & \bf{0.2462$\pm$0.0215} \\
					llog      & 0.8526$\pm$0.0302$\bullet$ & 0.8451$\pm$0.0334$\bullet$ & 0.8476$\pm$0.0353$\bullet$ & 0.8418$\pm$0.0428$\bullet$ & 0.8460$\pm$0.0388$\bullet$ & 0.7371$\pm$0.0423 & 0.7294$\pm$0.0282 & 0.9900$\pm$0.0096$\bullet$ & \bf{0.6898$\pm$0.0352} \\
					bibtext   & 0.3950$\pm$0.0134$\bullet$ & 0.3996$\pm$0.0088$\bullet$ & 0.3952$\pm$0.0087$\bullet$ & 0.3978$\pm$0.0102$\bullet$ & 0.3998$\pm$0.0084$\bullet$ & 0.3728$\pm$0.0082 & 0.3542$\pm$0.0128 & \bf{0.3522$\pm$0.0094} & 0.3656$\pm$0.0085 \\
					eurlex-sm & 0.6600$\pm$0.0253$\bullet$ & 0.6576$\pm$0.0239$\bullet$ & 0.6576$\pm$0.0239$\bullet$ & 0.6592$\pm$0.0253$\bullet$ & 0.6586$\pm$0.0259$\bullet$ & 0.4059$\pm$0.0267$\bullet$ & 0.2626$\pm$0.0234$\bullet$ & 0.6326$\pm$0.0186$\bullet$ & \bf{0.2306$\pm$0.0105} \\
					bookmark  & 0.7156$\pm$0.0051$\bullet$ & 0.7188$\pm$0.0098$\bullet$ & 0.7234$\pm$0.0062$\bullet$ & 0.7200$\pm$0.0072$\bullet$ & 0.7186$\pm$0.0061$\bullet$ & 0.6650$\pm$0.0095$\bullet$ & 0.6642$\pm$0.0066$\bullet$ & \bf{0.5800$\pm$0.0142} & 0.6198$\pm$0.0151 \\
					corel5k   & 0.6464$\pm$0.0084 & 0.6532$\pm$0.0093 & 0.6534$\pm$0.0086 & 0.6448$\pm$0.0100 & 0.6468$\pm$0.0107 & 0.6454$\pm$0.0098 & 0.6442$\pm$0.0114 & 0.8620$\pm$0.0110$\bullet$ & \bf{0.6370$\pm$0.0158} \\
					eurlex-dc  & 0.6944$\pm$0.0329$\bullet$ & 0.6914$\pm$0.0339$\bullet$ & 0.6914$\pm$0.0339$\bullet$ & 0.6140$\pm$0.0266$\bullet$ & 0.5916$\pm$0.0332$\bullet$ & 0.4458$\pm$0.0131$\bullet$ & 0.3450$\pm$0.0246$\bullet$ & 0.7312$\pm$0.0062$\bullet$ & \bf{0.3028$\pm$0.0189} \\
					espgame   & 0.5622$\pm$0.0166 & 0.5622$\pm$0.0166 & 0.5624$\pm$0.0173 & 0.5618$\pm$0.0168 & 0.5622$\pm$0.0166 & 0.5626$\pm$0.0162 & 0.5624$\pm$0.0161 & 0.9911$\pm$0.0031$\bullet$ & \bf{0.5604$\pm$0.0152} \\ 
					Delicious  & 0.3687$\pm$0.0127$\bullet$ & 0.3691$\pm$0.0093$\bullet$ & 0.3688$\pm$0.0086$\bullet$ & 0.3692$\pm$0.0125$\bullet$ & 0.3510$\pm$0.0123 & 0.3542$\pm$0.0122 & 0.3583$\pm$0.0131 & 0.4069$\pm$0.0162$\bullet$ & \bf{0.3386$\pm$0.0135} \\ 
					Mediamill  & 0.1337$\pm$0.0032 & 0.1322$\pm$0.0040 & 0.1320$\pm$0.0036 & 0.1317$\pm$0.0045 & 0.1319$\pm$0.0044 & 0.1315$\pm$0.0037 & 0.1337$\pm$0.0032 & 0.1374$\pm$0.0035$\bullet$ & \bf{0.1307$\pm$0.0035} \\ \midrule\midrule
					
					& \multicolumn{9}{c}{Precision$@3$ $\uparrow$}  \\ \hline
					Methods & ORI & PLST & CPLST & DMLR & FaIE & CMLL$_{\bm{y}}$ & MDDM & POP & CMLL \\ \hline
					Delicious & 0.5676$\pm$0.0055$\bullet$ & 0.5613$\pm$0.0042$\bullet$ & 0.5617$\pm$0.0046$\bullet$ & 0.5672$\pm$0.0052$\bullet$ & 0.5666$\pm$0.0050$\bullet$ & 0.5737$\pm$0.0051 & 0.5774$\pm$0.0051 & 0.4350$\pm$0.0100$\bullet$ & \bf{0.5873$\pm$0.0053} \\ 
					Mediamill & 0.6682$\pm$0.0036 & 0.6700$\pm$0.0033 & 0.6697$\pm$0.0031 & 0.6699$\pm$0.0033 & 0.6697$\pm$0.0034 & 0.6692$\pm$0.0039 & 0.6682$\pm$0.0035 & \bf{0.7633$\pm$0.0053}$\circ$ & 0.6722$\pm$0.0032 \\  \midrule\midrule
					
					& \multicolumn{9}{c}{nDCG$@3$ $\uparrow$}  \\ \hline
					Methods & ORI & PLST & CPLST & DMLR & FaIE & CMLL$_{\bm{y}}$ & MDDM & POP & CMLL \\ \hline
					Delicious & 0.5805$\pm$0.0069$\bullet$ & 0.5833$\pm$0.0047 & 0.5836$\pm$0.0050 & 0.5801$\pm$0.0066$\bullet$ & 0.5791$\pm$0.0063$\bullet$ & 0.5764$\pm$0.0065$\bullet$ & 0.5805$\pm$0.0067$\bullet$ & 0.4471$\pm$0.0113$\bullet$ & \bf{0.5904$\pm$0.0067} \\ 
					Mediamill & 0.7508$\pm$0.0036 & 0.7527$\pm$0.0037 & 0.7525$\pm$0.0034 & 0.7527$\pm$0.0038 & 0.7525$\pm$0.0039 & 0.7521$\pm$0.0041 & 0.7508$\pm$0.0036 & \bf{0.8234$\pm$0.0051}$\circ$ & 0.7536$\pm$0.0034 \\ \bottomrule
			\end{tabular}}
		\end{minipage}
	\end{center}
\end{table*}

\begin{table*} [tp]
	\huge
	\begin{center}
		\caption{Experimental results of k-CMLL with baselines.} \label{tb:kernel}
		\begin{minipage}{1\linewidth}
			\centering
			\vskip 0.15in
			\resizebox{\textwidth}{9cm}{
				\begin{tabular}{c|ccccccccc}
					\toprule
					& \multicolumn{8}{c}{Average Precision $\uparrow$}  \\ \midrule 
					Methods & k-ORI & k-CPLST & k-DMLR & k-FaIE & k-CMLL$_{\bm{y}}$ & k-MDDM & C2AE & k-CMLL \\ \hline
					plant             & 0.5894$\pm$0.0385 & 0.5907$\pm$0.0375 & 0.5985$\pm$0.0383 & 0.5917$\pm$0.0388 & 0.5928$\pm$0.0387 & 0.6185$\pm$0.0250 & 0.6277$\pm$0.0280 & \bf{0.6460$\pm$0.0396} \\
					msra              & 0.8087$\pm$0.0107 & 0.8090$\pm$0.0105 & 0.8282$\pm$0.0073 & \bf{0.8231$\pm$0.0090} & 0.8174$\pm$0.0068 & 0.8197$\pm$0.0111 & 0.8135$\pm$0.0100 & 0.8209$\pm$0.0103 \\
					enron             & 0.7001$\pm$0.0180 & 0.7005$\pm$0.0177 & 0.7001$\pm$0.0180 & 0.6722$\pm$0.0166 & 0.6930$\pm$0.0155 & 0.7091$\pm$0.0151 & 0.6856$\pm$0.0535 & \bf{0.7125$\pm$0.0116} \\
					llog              & 0.4269$\pm$0.0253$\bullet$ & 0.4269$\pm$0.0253$\bullet$ & 0.4269$\pm$0.0273$\bullet$ & 0.4276$\pm$0.0251$\bullet$ & 0.4290$\pm$0.0254$\bullet$ & 0.4675$\pm$0.0280 & 0.3770$\pm$0.0130$\bullet$ & \bf{0.4755$\pm$0.0207} \\
					bibtext           & 0.5957$\pm$0.0051$\bullet$ & \bf{0.6220$\pm$0.0069}$\circ$ & 0.5948$\pm$0.0064$\bullet$ & 0.5969$\pm$0.0058$\bullet$ & 0.5970$\pm$0.0059$\bullet$ & 0.6030$\pm$0.0060 & 0.6204$\pm$0.0072$\circ$ & 0.6060$\pm$0.0042 \\
					eurlex-sm         & 0.8011$\pm$0.0157$\bullet$ & 0.8006$\pm$0.0170$\bullet$ & 0.8011$\pm$0.0157$\bullet$ & 0.8112$\pm$0.0158$\bullet$ & 0.8010$\pm$0.0157$\bullet$ & 0.8060$\pm$0.0161 & 0.7773$\pm$0.0359$\bullet$ & \bf{0.8265$\pm$0.0150} \\
					bookmark          & 0.4067$\pm$0.0088$\bullet$ & 0.4067$\pm$0.0088$\bullet$ & 0.4129$\pm$0.0105$\bullet$ & 0.4070$\pm$0.0089$\bullet$ & 0.4072$\pm$0.0088$\bullet$ & 0.4288$\pm$0.0052$\bullet$ & 0.3999$\pm$0.0106$\bullet$ & \bf{0.4582$\pm$0.0074} \\
					corel5k           & 0.3035$\pm$0.0050$\bullet$ & 0.3036$\pm$0.0050$\bullet$ & 0.3083$\pm$0.0055$\bullet$ & 0.3038$\pm$0.0046$\bullet$ & 0.3034$\pm$0.0049$\bullet$ & 0.3307$\pm$0.0060 & 0.3169$\pm$0.0019$\bullet$ & \bf{0.3321$\pm$0.0084} \\
					eurlex-dc          & 0.7578$\pm$0.0079$\bullet$ & 0.7547$\pm$0.0076$\bullet$ & 0.7576$\pm$0.0079$\bullet$ & 0.7585$\pm$0.0077$\bullet$ & 0.7580$\pm$0.0077$\bullet$ & 0.7521$\pm$0.0102$\bullet$ & 0.7411$\pm$0.0180$\bullet$ & \bf{0.7799$\pm$0.0130} \\
					espgame           & 0.2298$\pm$0.0065 & 0.2298$\pm$0.0064 & 0.2298$\pm$0.0065 & 0.2293$\pm$0.0065 & 0.2293$\pm$0.0063 & 0.2336$\pm$0.0069 & 0.2286$\pm$0.0022 & \bf{0.2346$\pm$0.0072} \\
					Delicious         & 0.3576$\pm$0.0198$\bullet$ & 0.3567$\pm$0.0099$\bullet$ & 0.3702$\pm$0.0174$\bullet$ & 0.3527$\pm$0.0202$\bullet$ & 0.3587$\pm$0.0184$\bullet$ & 0.3742$\pm$0.0236 & 0.3596$\pm$0.0025$\bullet$ & \bf{0.3863$\pm$0.0041} \\
					Mediamill         & 0.7226$\pm$0.1789 & 0.7213$\pm$0.0282$\bullet$ & 0.7696$\pm$0.1142 & 0.7063$\pm$0.0399$\bullet$ & 0.7811$\pm$0.1248 & \bf{0.7857$\pm$0.0864} & 0.6861$\pm$0.0073$\bullet$ & 0.7802$\pm$0.0711 \\ \midrule \midrule 
					
					& \multicolumn{9}{c}{micro-F1 $\uparrow$}  \\ \midrule 
					Methods & k-ORI & k-CPLST & k-DMLR & k-FaIE & k-CMLL$_{\bm{y}}$ & k-MDDM & C2AE & k-CMLL \\ \hline
					plant     & 0.3237$\pm$0.0447 & 0.3268$\pm$0.0432 & 0.3284$\pm$0.0479 & 0.3261$\pm$0.0468 & 0.3280$\pm$0.0432 & 0.3430$\pm$0.0454 & \bf{0.3680$\pm$0.034} & 0.3526$\pm$0.0294 \\
					msra      & 0.6711$\pm$0.0083$\bullet$ & 0.6712$\pm$0.0084$\bullet$ & 0.6837$\pm$0.0081$\bullet$ & 0.6840$\pm$0.0088$\bullet$ & 0.6831$\pm$0.0074$\bullet$ & 0.6889$\pm$0.0114 & 0.6708$\pm$0.0084$\bullet$ & \bf{0.7004$\pm$0.0114} \\
					enron     & 0.5849$\pm$0.0062$\bullet$ & 0.5852$\pm$0.0064$\bullet$ & 0.5849$\pm$0.0062$\bullet$ & 0.5857$\pm$0.0060$\bullet$ & 0.5851$\pm$0.0060$\bullet$ & 0.6037$\pm$0.0119$\bullet$ & 0.6512$\pm$0.0310 & \bf{0.6582$\pm$0.0115} \\
					llog      & 0.1512$\pm$0.0183 & 0.1512$\pm$0.0183 & 0.1512$\pm$0.0182 & 0.1513$\pm$0.0185 & 0.1612$\pm$0.0184 & 0.1566$\pm$0.0247 & \bf{0.2829$\pm$0.0199}$\circ$ & 0.1733$\pm$0.0150 \\
					bibtext   & 0.3512$\pm$0.0075$\bullet$ & 0.3680$\pm$0.0104$\bullet$ & 0.3489$\pm$0.0064$\bullet$ & 0.3520$\pm$0.0055$\bullet$ & 0.3521$\pm$0.0055$\bullet$ & 0.3985$\pm$0.0069 & 0.3996$\pm$0.0081 & \bf{0.4069$\pm$0.0075} \\
					eurlex-sm & 0.5580$\pm$0.0205$\bullet$ & 0.5578$\pm$0.0205$\bullet$ & 0.5580$\pm$0.0205$\bullet$ & 0.5586$\pm$0.0204$\bullet$ & 0.5581$\pm$0.0205$\bullet$ & 0.6564$\pm$0.0159 & 0.6061$\pm$0.0106$\bullet$ & \bf{0.6555$\pm$0.0164} \\
					bookmark  & 0.2019$\pm$0.0057$\bullet$ & 0.2019$\pm$0.0057$\bullet$ & 0.2268$\pm$0.0075$\bullet$ & 0.2020$\pm$0.0056$\bullet$ & 0.2027$\pm$0.0057$\bullet$ & 0.2106$\pm$0.0063$\bullet$ & \bf{0.2657$\pm$0.0103} & 0.2378$\pm$0.0053 \\
					corel5k   & 0.1146$\pm$0.0056$\bullet$ & 0.1146$\pm$0.0056$\bullet$ & 0.1166$\pm$0.0050$\bullet$ & 0.1149$\pm$0.0054$\bullet$ & 0.1147$\pm$0.0057$\bullet$ & 0.1431$\pm$0.0012$\bullet$ & 0.1685$\pm$0.0057 & \bf{0.1702$\pm$0.0064} \\
					eurlex-dc  & 0.4648$\pm$0.0160$\bullet$ & 0.4647$\pm$0.0160$\bullet$ & 0.4663$\pm$0.0161$\bullet$ & 0.4659$\pm$0.0161$\bullet$ & 0.4651$\pm$0.0161$\bullet$ & 0.5489$\pm$0.0199 & 0.4847$\pm$0.0015$\bullet$ & \bf{0.5554$\pm$0.0181} \\
					espgame   & 0.1039$\pm$0.0049 & 0.1040$\pm$0.0047 & 0.1039$\pm$0.0049 & 0.1040$\pm$0.0048 & 0.1039$\pm$0.0049 & 0.1062$\pm$0.0064 & \bf{0.1178$\pm$0.0131}$\circ$ & 0.1063$\pm$0.0078 \\
					Delicious & 0.1795$\pm$0.0133$\bullet$ & 0.1780$\pm$0.0073$\bullet$ & 0.1876$\pm$0.0162$\bullet$ & 0.1851$\pm$0.0199$\bullet$ & 0.2050$\pm$0.0220$\bullet$ & 0.1900$\pm$0.0034$\bullet$ & \bf{0.3416$\pm$0.0019}$\circ$ & 0.2361$\pm$0.0029$\bullet$ \\
					Mediamill & 0.5405$\pm$0.0983$\bullet$ & 0.5415$\pm$0.0382$\bullet$ & 0.5611$\pm$0.0675 & 0.5331$\pm$0.0263$\bullet$ & 0.5429$\pm$0.1306 & 0.5913$\pm$0.0247 & 0.5556$\pm$0.0049$\bullet$ & \bf{0.5927$\pm$0.0199} \\ \midrule \midrule 
					
					& \multicolumn{8}{c}{Ranking Loss $\downarrow$}  \\ \midrule 
					Methods & k-ORI & k-CPLST & k-DMLR & k-FaIE & k-CMLL$_{\bm{y}}$ & k-MDDM & C2AE & k-CMLL \\ \hline
					plant        & 0.1771$\pm$0.0312 & 0.1756$\pm$0.0300 & 0.1701$\pm$0.0257 & 0.1761$\pm$0.0314 & 0.1756$\pm$0.0312 & 0.1666$\pm$0.0232 & 0.1578$\pm$0.0194 & \bf{0.1495$\pm$0.0294} \\
					msra         & 0.1435$\pm$0.0075$\bullet$ & 0.1433$\pm$0.0073$\bullet$ & 0.1176$\pm$0.0058$\circ$ & 0.1234$\pm$0.0066 & 0.1182$\pm$0.0058 & 0.1275$\pm$0.0087 & \bf{0.1113$\pm$0.0073}$\circ$ & 0.1289$\pm$0.0088 \\
					enron        & 0.0973$\pm$0.0138$\bullet$ & 0.0969$\pm$0.0138$\bullet$ & 0.0973$\pm$0.0138$\bullet$ & 0.1012$\pm$0.0129$\bullet$ & 0.1000$\pm$0.0128$\bullet$ & 0.1011$\pm$0.0120$\bullet$ & 0.0833$\pm$0.0139 & \bf{0.0816$\pm$0.0037} \\
					llog         & 0.1758$\pm$0.0293$\bullet$ & 0.1757$\pm$0.0292$\bullet$ & 0.1787$\pm$0.0301$\bullet$ & 0.1764$\pm$0.0297$\bullet$ & 0.1762$\pm$0.0279$\bullet$ & 0.1467$\pm$0.0170 & \bf{0.1249$\pm$0.0141} & 0.1362$\pm$0.0149 \\
					bibtext      & 0.0939$\pm$0.0109 & 0.0980$\pm$0.0090$\bullet$ & 0.0933$\pm$0.0105 & 0.0936$\pm$0.0108 & 0.0934$\pm$0.0107 & 0.0819$\pm$0.0083 & \bf{0.0565$\pm$0.0721}$\circ$ & 0.0798$\pm$0.0091 \\
					eurlex-sm    & 0.0221$\pm$0.0084 & 0.0266$\pm$0.0078 & 0.0219$\pm$0.0085 & 0.0219$\pm$0.0084 & 0.0221$\pm$0.0086 & 0.0210$\pm$0.0049 & 0.0240$\pm$0.1171 & \bf{0.0202$\pm$0.0061} \\
					bookmark     & 0.1604$\pm$0.0071$\bullet$ & 0.1604$\pm$0.0071$\bullet$ & 0.1662$\pm$0.0088$\bullet$ & 0.1603$\pm$0.0072$\bullet$ & 0.1509$\pm$0.0075 & 0.1584$\pm$0.0060$\bullet$ & 0.1527$\pm$0.0279 & \bf{0.1447$\pm$0.0050} \\
					corel5k      & 0.1941$\pm$0.0068$\bullet$ & 0.1440$\pm$0.0068 & 0.1566$\pm$0.0078 & 0.1741$\pm$0.0070$\bullet$ & 0.1640$\pm$0.0068 & 0.1937$\pm$0.0052$\bullet$ & \bf{0.1321$\pm$0.0240} & 0.1523$\pm$0.0119\\
					eurlex-dc     & 0.0416$\pm$0.0042 & 0.0374$\pm$0.0053 & 0.0409$\pm$0.0041 & 0.0417$\pm$0.0039 & 0.0417$\pm$0.0042 & 0.0361$\pm$0.0084 & 0.0451$\pm$0.0418 & \bf{0.0357$\pm$0.0046} \\
					espgame      & 0.0253$\pm$0.0065 & 0.0251$\pm$0.0065 & 0.0256$\pm$0.0067 & 0.0255$\pm$0.0065 & 0.0256$\pm$0.0066 & 0.0249$\pm$0.0070 & \bf{0.0206$\pm$0.0123} & 0.0232$\pm$0.0071 \\
					Delicious    & 0.1746$\pm$0.0468 & 0.1630$\pm$0.0135 & 0.1701$\pm$0.0371 & 0.2053$\pm$0.0531 & 0.1684$\pm$0.0560 & 0.1689$\pm$0.0373 & \bf{0.1232$\pm$0.0012} & 0.1599$\pm$0.0318 \\
					Mediamill    & 0.0671$\pm$0.0027$\bullet$ & 0.0644$\pm$0.0051 & 0.0634$\pm$0.0014$\bullet$ & 0.0663$\pm$0.0054 & 0.0624$\pm$0.0056 & 0.0666$\pm$0.0024$\bullet$ & 0.0659$\pm$0.0034$\bullet$ & \bf{0.0599$\pm$0.0021} \\ \midrule \midrule 
					
					& \multicolumn{8}{c}{One Error $\downarrow$}  \\ \midrule 
					Methods & k-ORI & k-CPLST & k-DMLR & k-FaIE & k-CMLL$_{\bm{y}}$ & k-MDDM & C2AE & k-CMLL \\ \hline
					plant     & 0.5857$\pm$0.0457 & 0.5836$\pm$0.0435 & 0.5826$\pm$0.0518 & 0.5867$\pm$0.0463 & 0.5847$\pm$0.0457 & 0.5458$\pm$0.0290 & 0.5581$\pm$0.0473 & \bf{0.5407$\pm$0.0510} \\
					msra      & 0.0653$\pm$0.0198 & 0.0642$\pm$0.0180 & 0.0525$\pm$0.0136 & 0.0553$\pm$0.0057 & 0.0583$\pm$0.0130 & 0.0637$\pm$0.0166 & 0.0617$\pm$0.0267 & \bf{0.0507$\pm$0.0099} \\
					enron     & 0.2290$\pm$0.0165 & 0.2278$\pm$0.0140 & 0.2290$\pm$0.0165 & 0.2454$\pm$0.0205$\bullet$ & 0.2331$\pm$0.0187 & 0.2213$\pm$0.0196 & 0.2562$\pm$0.0571 & \bf{0.2091$\pm$0.0169} \\
					llog      & 0.7145$\pm$0.0272$\bullet$ & 0.7145$\pm$0.0272$\bullet$ & 0.7153$\pm$0.0310$\bullet$ & 0.7145$\pm$0.0272$\bullet$ & 0.7154$\pm$0.0316$\bullet$ & 0.6756$\pm$0.0328 & 0.7642$\pm$0.0149$\bullet$ & \bf{0.6625$\pm$0.0374} \\
					bibtext   & 0.3506$\pm$0.0043$\bullet$ & 0.3496$\pm$0.0166 & 0.3442$\pm$0.0064 & 0.3492$\pm$0.0044$\bullet$ & 0.3392$\pm$0.0044 & 0.3416$\pm$0.0050 & \bf{0.3164$\pm$0.0092} & 0.3368$\pm$0.0090 \\
					eurlex-sm & 0.1754$\pm$0.0198$\bullet$ & 0.1788$\pm$0.0204$\bullet$ & 0.1754$\pm$0.0198$\bullet$ & 0.1756$\pm$0.0200$\bullet$ & 0.1756$\pm$0.0200$\bullet$ & 0.1514$\pm$0.0282 & 0.1293$\pm$0.0116 & \bf{0.1261$\pm$0.0192} \\
					bookmark  & 0.6008$\pm$0.0077$\bullet$ & 0.6008$\pm$0.0077$\bullet$ & 0.5928$\pm$0.0141$\bullet$ & 0.6006$\pm$0.0078$\bullet$ & 0.6000$\pm$0.0080$\bullet$ & 0.5724$\pm$0.0084 & \bf{0.5430$\pm$0.0097} & 0.5616$\pm$0.0090 \\
					corel5k   & 0.6274$\pm$0.0111$\bullet$ & 0.6268$\pm$0.0112$\bullet$ & 0.6286$\pm$0.0087$\bullet$ & 0.6282$\pm$0.0110$\bullet$ & 0.6278$\pm$0.0112$\bullet$ & 0.6138$\pm$0.0138$\bullet$ & \bf{0.5876$\pm$0.0142} & 0.5880$\pm$0.0169 \\
					eurlex-dc  & 0.3098$\pm$0.0112$\bullet$ & 0.3140$\pm$0.0095$\bullet$ & 0.3096$\pm$0.0115$\bullet$ & 0.3088$\pm$0.0110$\bullet$ & 0.3096$\pm$0.0113$\bullet$ & 0.2838$\pm$0.0137 & 0.3012$\pm$0.0208$\bullet$ & \bf{0.2800$\pm$0.0163} \\
					espgame   & 0.5458$\pm$0.0153 & 0.5462$\pm$0.0157 & 0.5460$\pm$0.0134 & 0.5464$\pm$0.0140 & 0.5456$\pm$0.0148 & 0.5363$\pm$0.0123 & 0.5325$\pm$0.0425 & \bf{0.5306$\pm$0.0091} \\
					Delicious & 0.3446$\pm$0.0640 & 0.3320$\pm$0.0321 & 0.3236$\pm$0.0515 & 0.3257$\pm$0.0464 & 0.3336$\pm$0.0258 & 0.3345$\pm$0.1174 & 0.3497$\pm$0.0058 & \bf{0.3260$\pm$0.0617} & \\
					Mediamill & 0.1567$\pm$0.0045$\bullet$ & 0.1386$\pm$0.0029 & 0.1439$\pm$0.0099$\bullet$ & 0.1510$\pm$0.0039$\bullet$ & 0.1457$\pm$0.0020$\bullet$ & 0.1412$\pm$0.0072 & 0.1507$\pm$0.0131 & \bf{0.1384$\pm$0.0040} \\ \midrule \midrule 
					
					& \multicolumn{8}{c}{Precision$@3$ $\uparrow$}  \\ \hline
					Methods & k-ORI & k-CPLST & k-DMLR & k-FaIE & k-CMLL$_{\bm{y}}$ & k-MDDM & C2AE & k-CMLL \\ \hline
					Delicious    & 0.5739$\pm$0.0045$\bullet$ & 0.5814$\pm$0.0029$\bullet$ & 0.5882$\pm$0.0049$\bullet$ & 0.5859$\pm$0.0081$\bullet$ & 0.5887$\pm$0.0026$\bullet$ & 0.5800$\pm$0.0033$\bullet$ & 0.5935$\pm$0.0059$\bullet$ & \bf{0.6090$\pm$0.0022} \\
					Mediamill    & 0.6480$\pm$0.0062$\bullet$ & 0.6576$\pm$0.0080 & 0.6599$\pm$0.0092 & 0.6400$\pm$0.0032$\bullet$ & 0.6582$\pm$0.0080 & 0.6605$\pm$0.0095 & 0.6492$\pm$0.0038$\bullet$ & \bf{0.6622$\pm$0.0072} \\ \midrule\midrule
					
					& \multicolumn{8}{c}{nDCG$@3$ $\uparrow$}  \\ \hline
					& k-ORI & k-CPLST & k-DMLR & k-FaIE & k-CMLL$_{\bm{y}}$ & k-MDDM & C2AE & k-CMLL \\ \hline
					Delicious & 0.5911$\pm$0.0087 & 0.5901$\pm$0.0034$\bullet$ & 0.5970$\pm$0.0042 & 0.5944$\pm$0.0097 & 0.5950$\pm$0.0022$\bullet$ & 0.5952$\pm$0.0058 & 0.6073$\pm$0.0060 & \bf{0.6098$\pm$0.0075} \\
					Mediamill & 0.7360$\pm$0.0098 & 0.7361$\pm$0.0039 & 0.7349$\pm$0.0066 & 0.7342$\pm$0.0043$\bullet$ & 0.7489$\pm$0.0021 & 0.7479$\pm$0.0077 & 0.7285$\pm$0.0034$\bullet$ & \bf{0.7496$\pm$0.0034}  \\ \bottomrule
			\end{tabular}}
		\end{minipage}
	\end{center}
\end{table*}

\begin{figure}[!t]
	\begin{center}
		\begin{minipage}[t]{0.48\columnwidth}
			\centering
			\includegraphics[width=1\columnwidth]{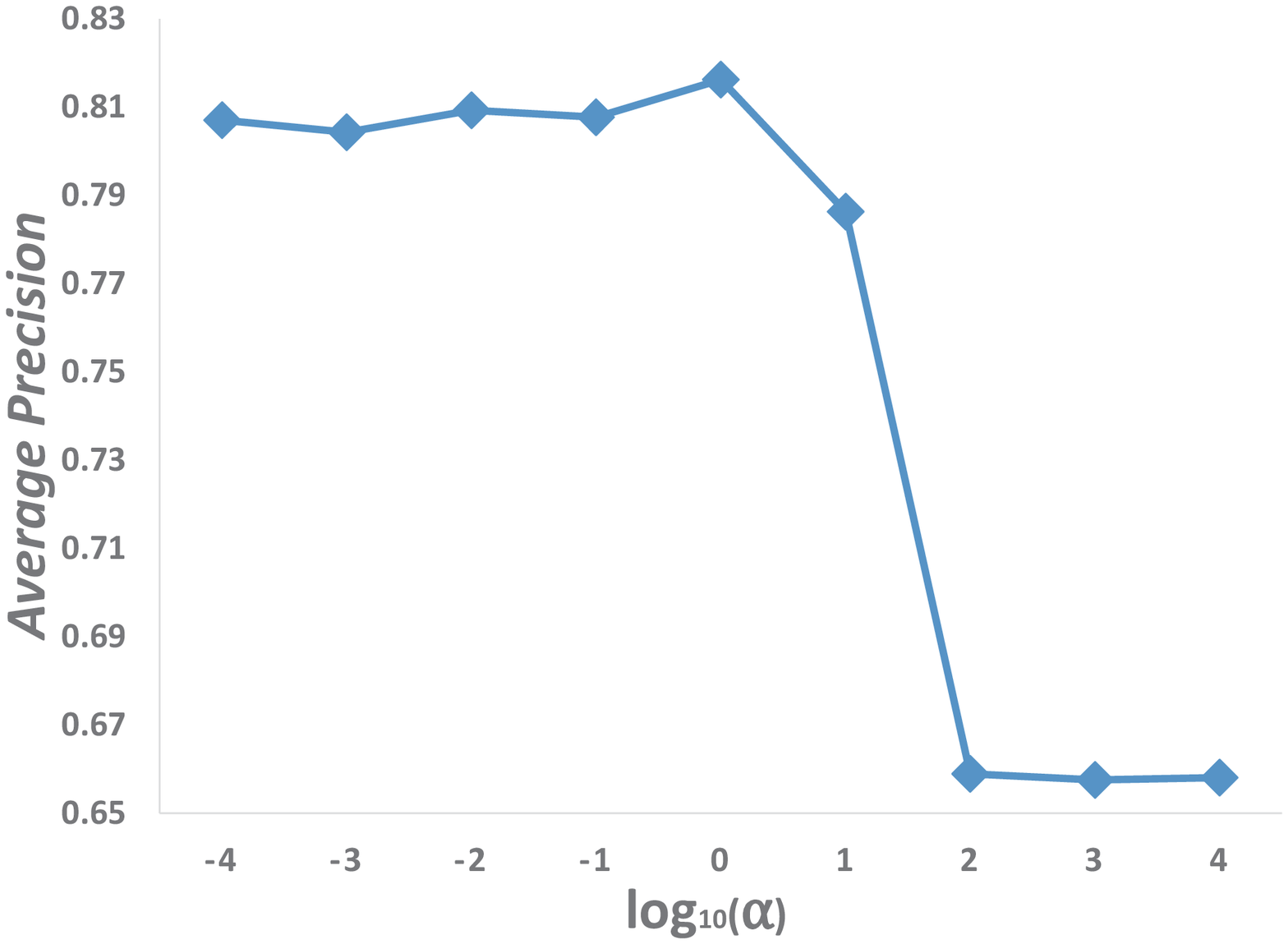}
			{(a) ~Average Precision}
		\end{minipage}
		\begin{minipage}[t]{0.48\columnwidth}
			\centering
			\includegraphics[width=1\columnwidth]{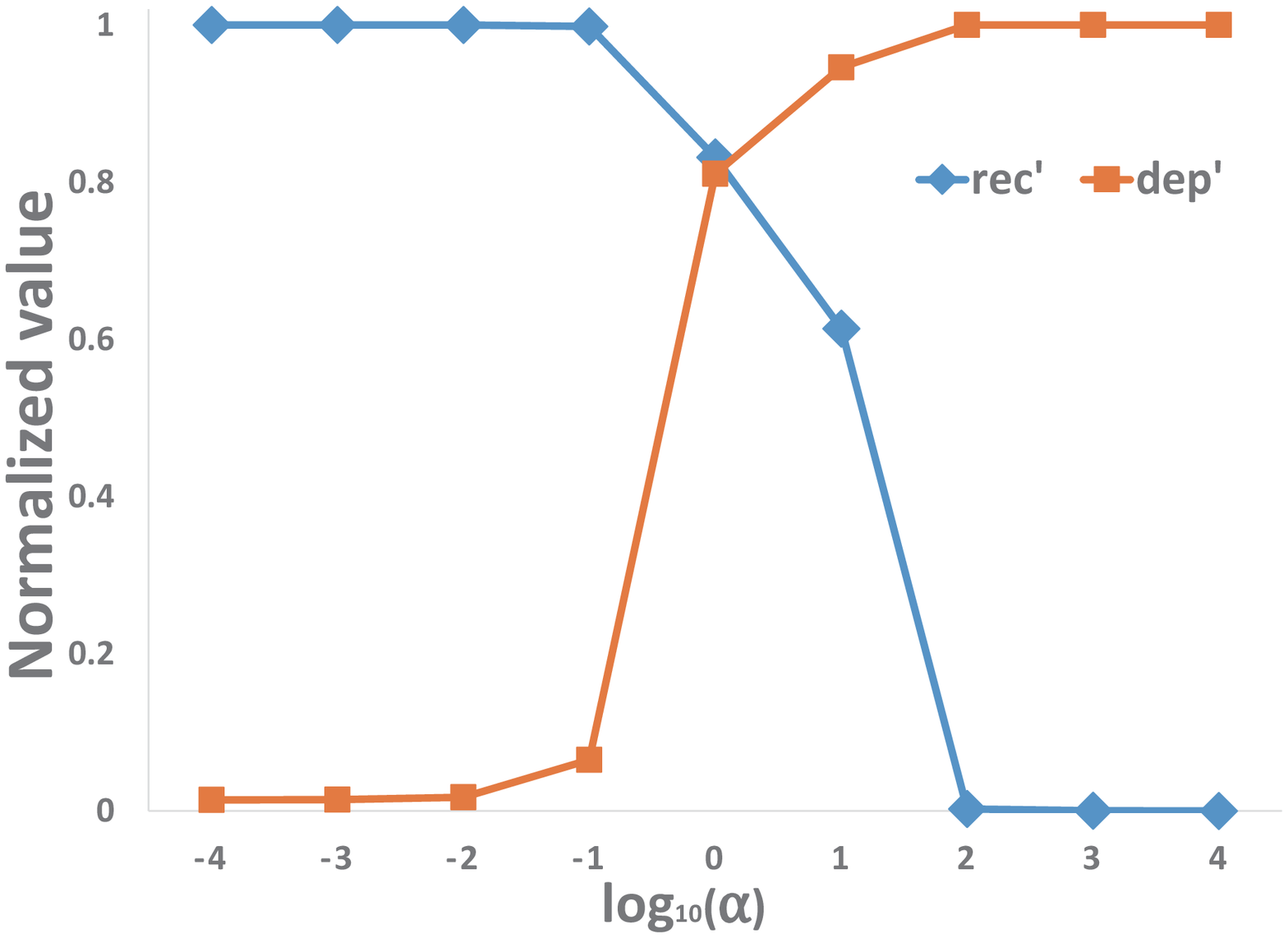}
			{(b) ~$dep'$ and $rec'$}
		\end{minipage}
		\caption{Average precision changes as parameter $\alpha$ varies on msra.}  \label{fig:alpha}
	\end{center}
	\vspace{-1ex}
\end{figure}

\begin{figure*}[!t]
	\begin{center}
		\begin{minipage}[t]{0.96\linewidth}  
			\centering
			\includegraphics[width=0.48\linewidth]{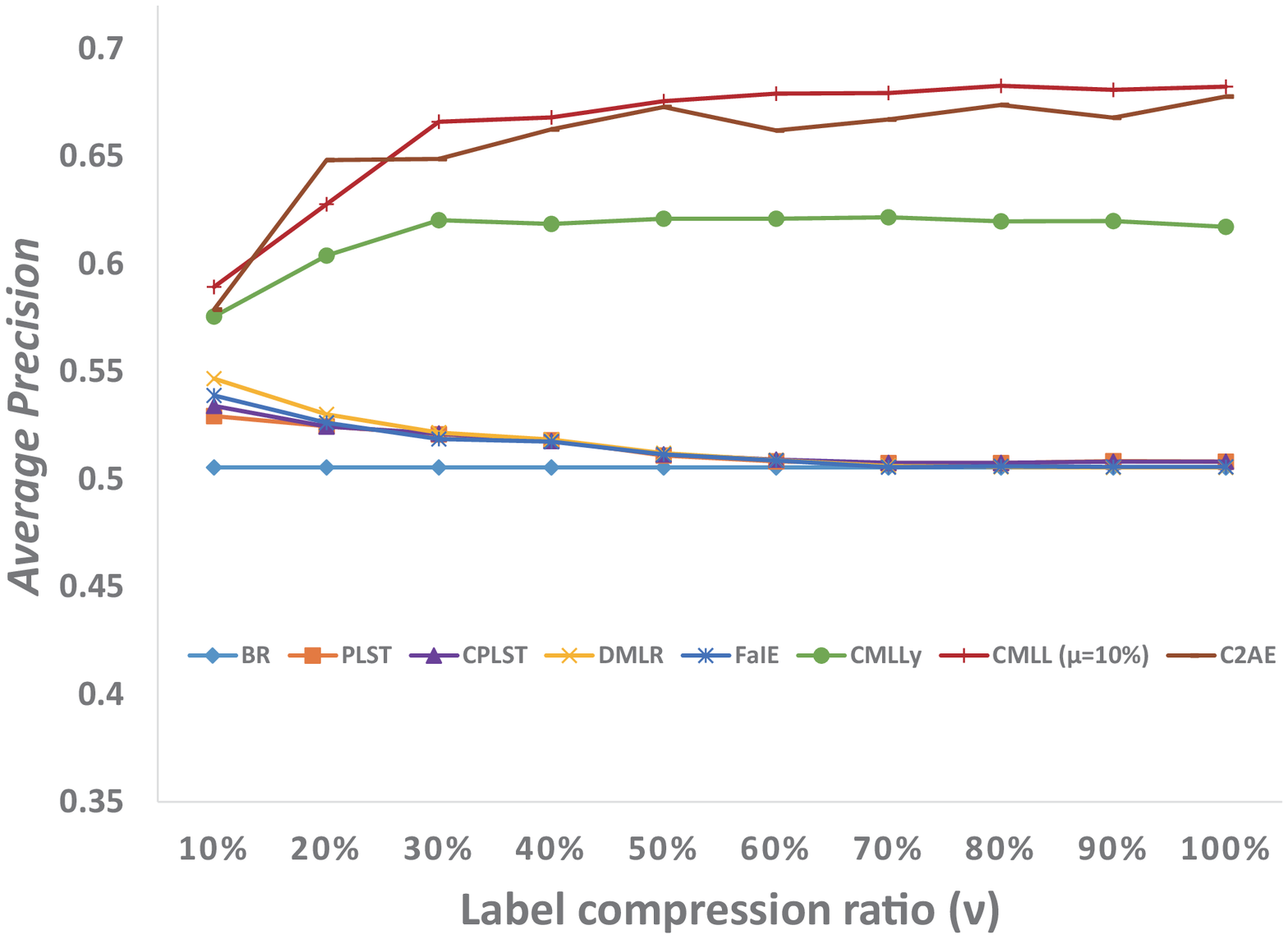}
			\includegraphics[width=0.48\linewidth]{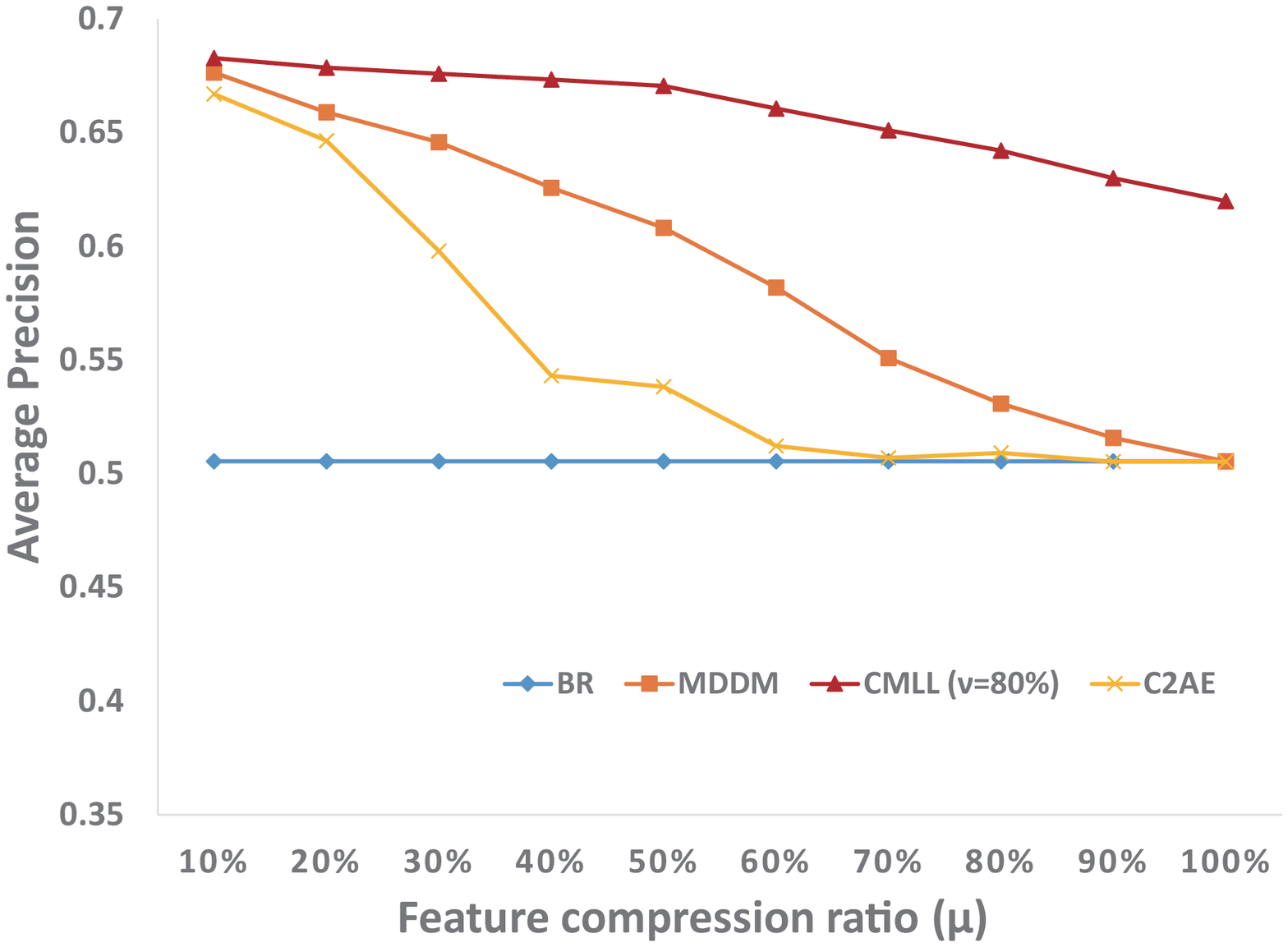}
		\end{minipage}
		\begin{minipage}[t]{0.96\linewidth}  
			\centering
			\includegraphics[width=0.48\linewidth]{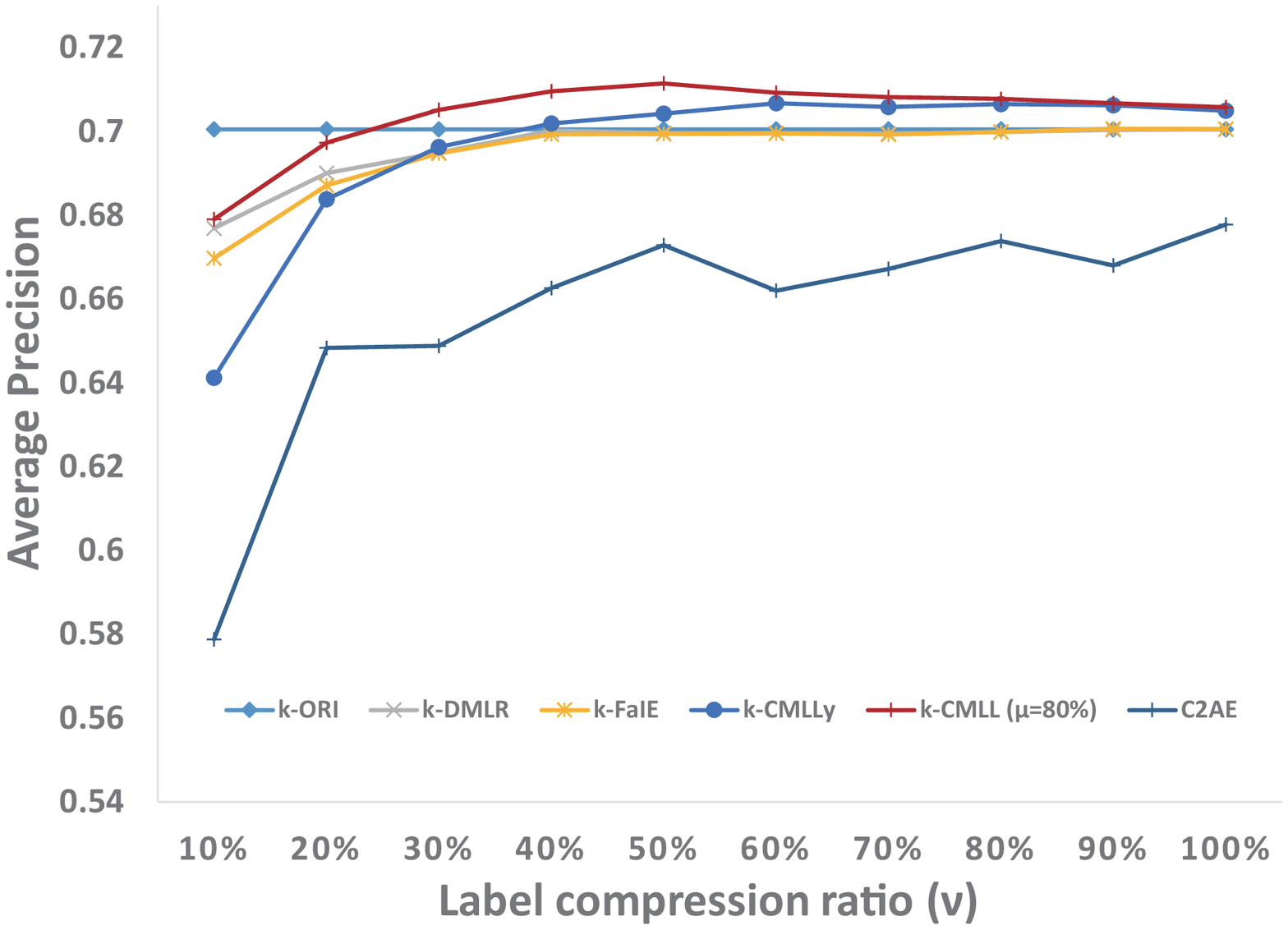}
			\includegraphics[width=0.48\linewidth]{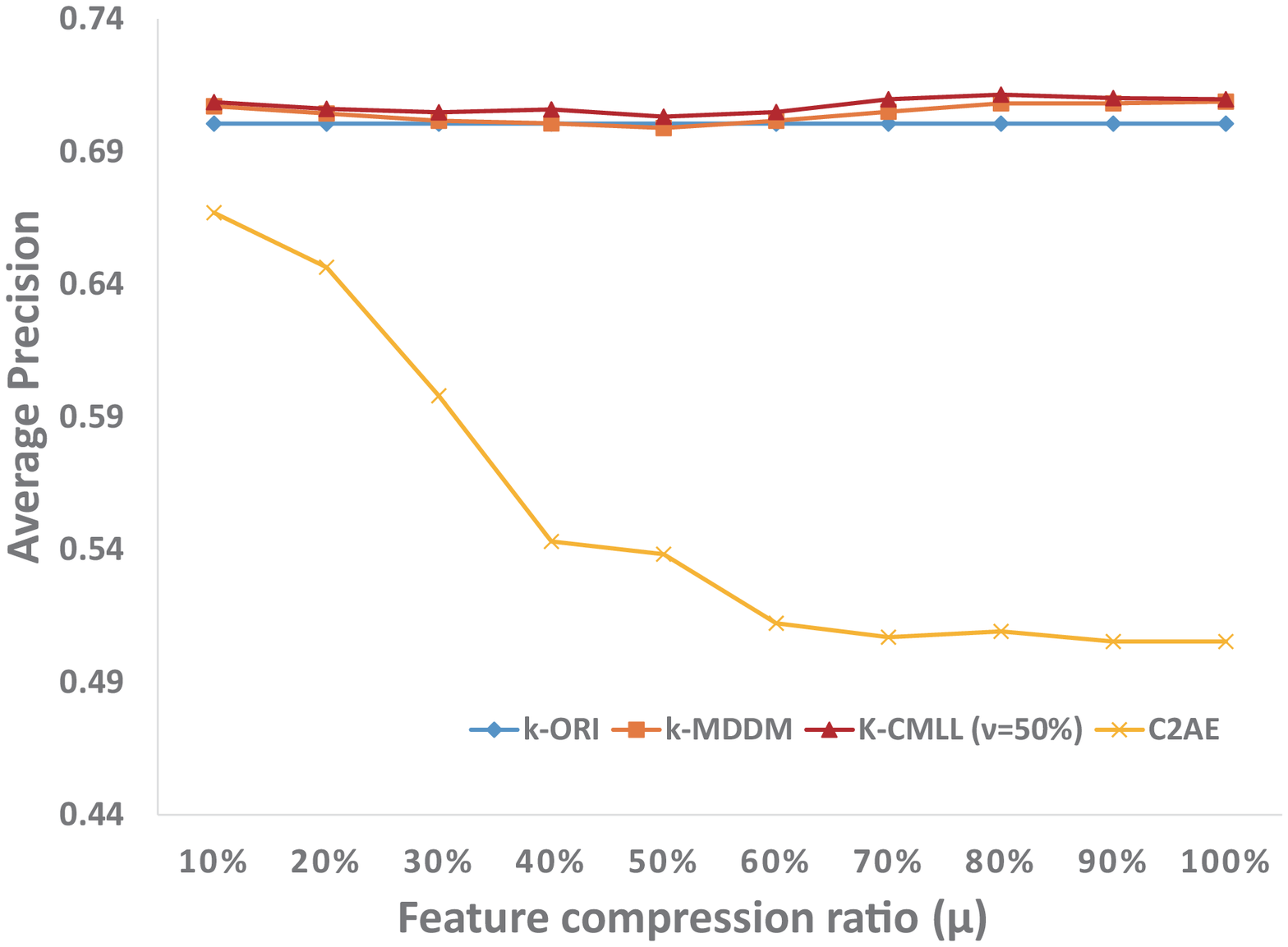}
		\end{minipage}
		
	\end{center}
	
	\begin{center}
		\vspace{-2ex}
		\caption{The average precision curve of moving the compression ratio on enron.}  \label{fig:enron_ap}
	\end{center}
	\vspace{-3ex}
\end{figure*}

\begin{figure*}[!t]
	\vspace{-2ex}
	\begin{center}
		\begin{minipage}[t]{0.48\linewidth}  
			\centering
			\includegraphics[width=1\linewidth]{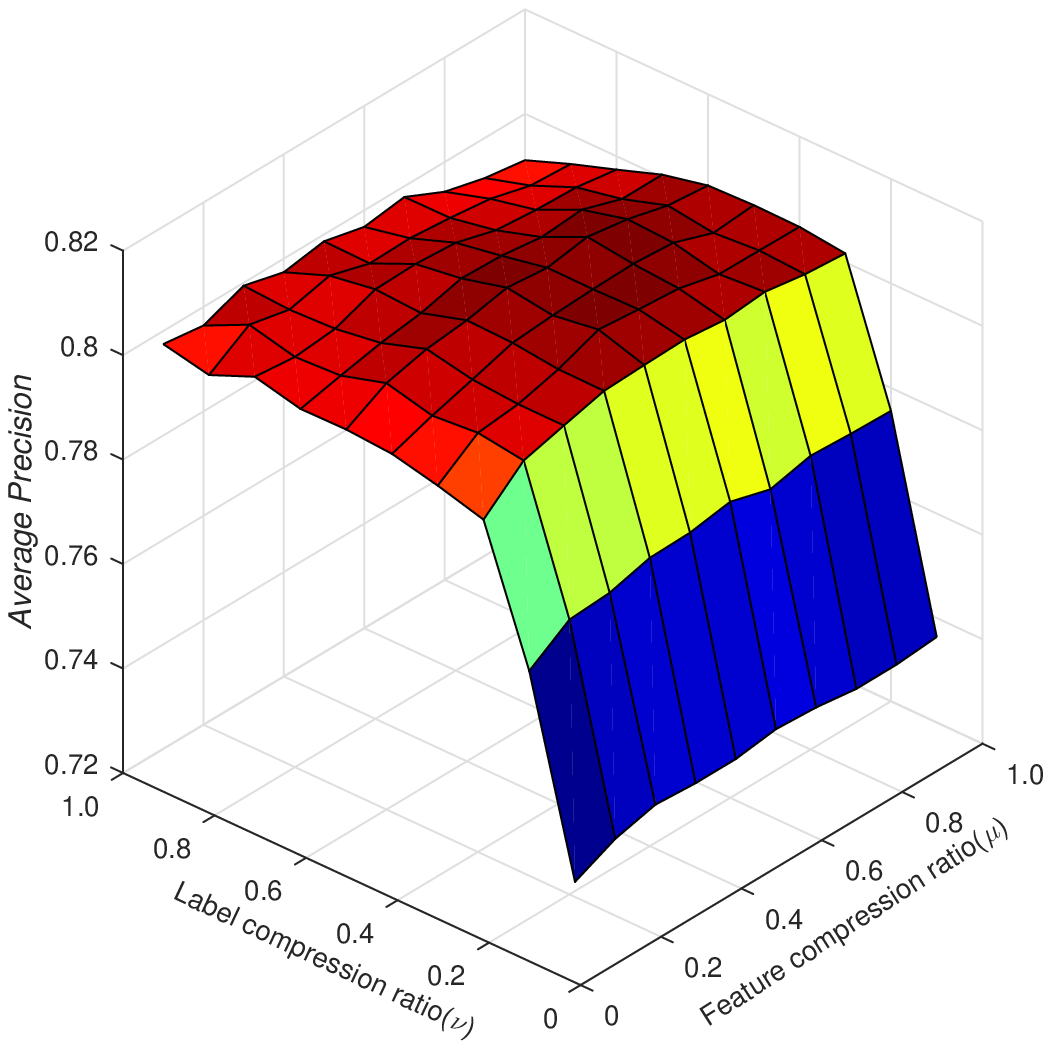}
			{(a) msra}
		\end{minipage}
		\begin{minipage}[t]{0.48\linewidth}  
			\centering
			\includegraphics[width=1\linewidth]{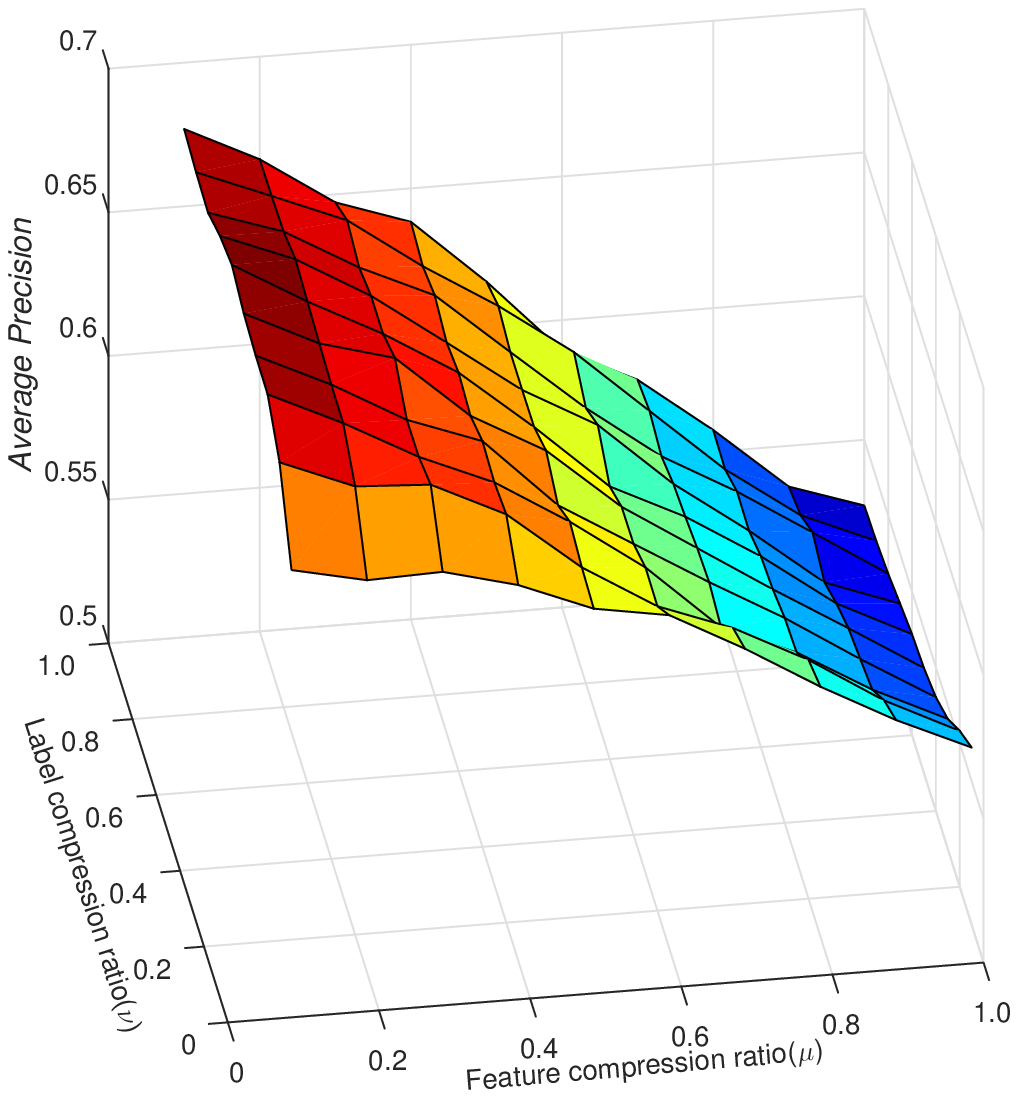}
			{(b) enron}
		\end{minipage}
		\begin{minipage}[t]{0.48\linewidth}  
			\centering
			\includegraphics[width=1\linewidth]{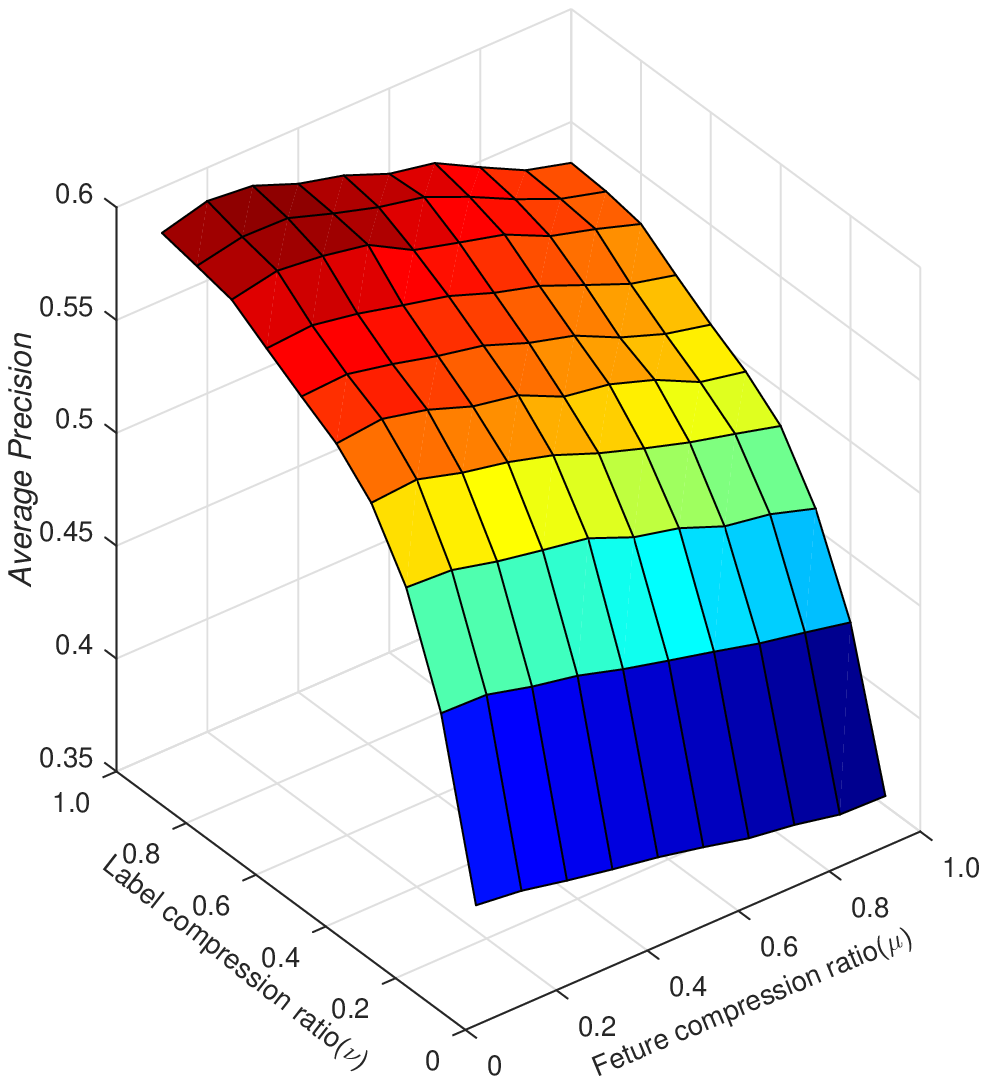}
			{(c) bibtex}
		\end{minipage}
		\begin{minipage}[t]{0.48\linewidth}  
			\centering
			\includegraphics[width=1\linewidth]{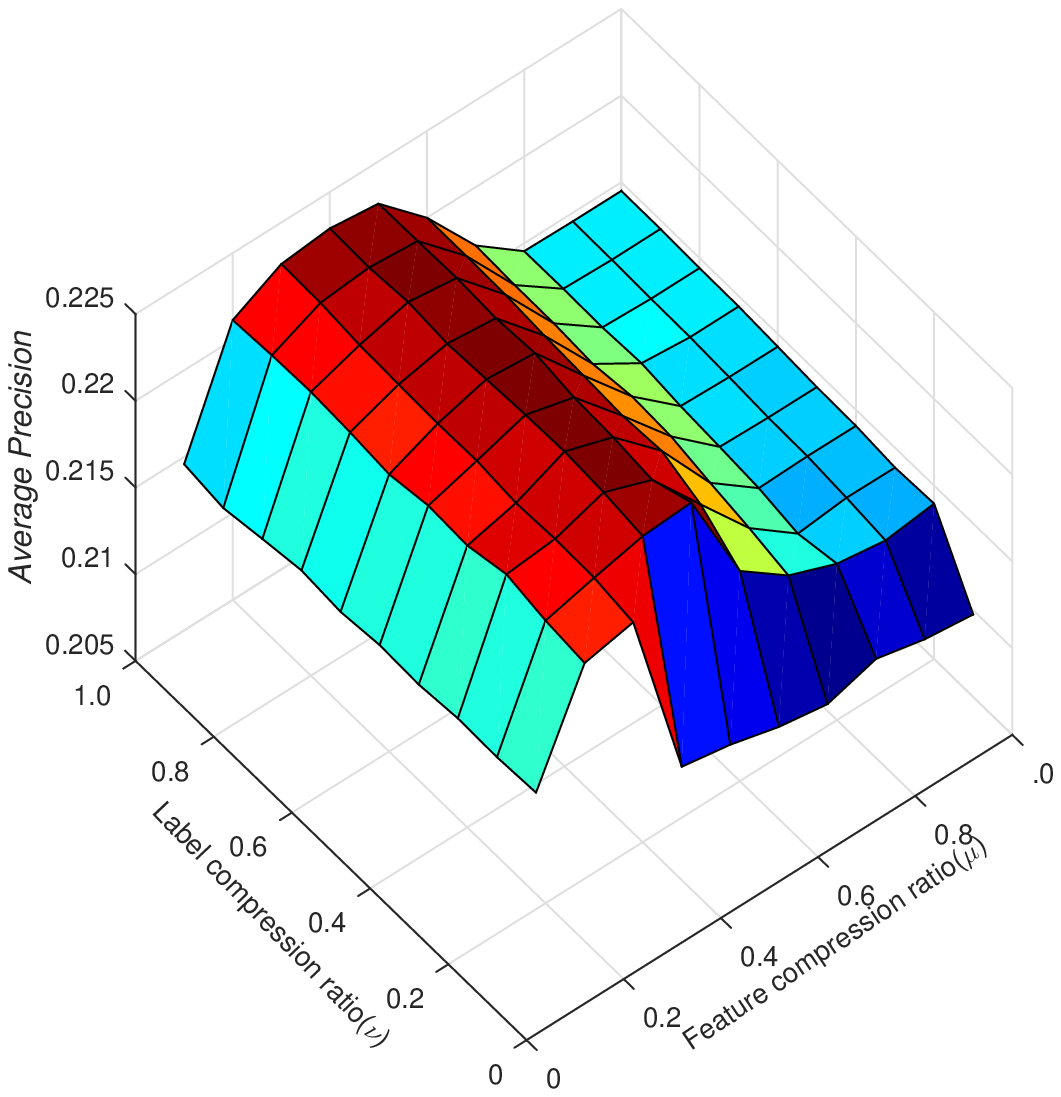}
			{(d) esmpgame}
		\end{minipage}
	\end{center}
	\caption{The spatial graphs of CMLL over various $\mu$ and $\nu$. }  \label{fig:3d}
\end{figure*}

\subsection{Results}

The experimental results of CMLL and k-CMLL compared with the comparing methods are shown in Table~\ref{tb:linear} and Table~\ref{tb:kernel} respectively. For each metric, ``$\downarrow$'' indicates the smaller the better while ``$\uparrow$'' indicates the larger the better. We perform five-fold cross-validation on each dataset, and use paired $t$-test at 10\% significance level. The mean results with standard deviation are reported and the best performance is highlighted in boldface. $\bullet/\circ$ represents whether CMLL or k-CMLL is significantly better/worse than the comparing methods. We can observe that across all metrics, CMLL ranks 1st in the most cases in both linear and non-linear cases.

As Table~\ref{tb:linear} and Table~\ref{tb:kernel} show, with a suitable compression ratio, most embedding methods can achieve better performance than the baseline ORI. This indicates that there are indeed some problems such as sparsity and noise existing in both the original spaces, which leads to performance decline if not tackled properly. With a more compact representation for labels ($\nu \le 100\%$) and features ($\mu \le 100\%$), CMLL performs better than all compared LC methods ($\mu=100\%$) and FE method ($\nu=100\%$) in most cases. By applying the kernel trick to extend methods to their corresponding non-linear version, each LC method actually guides the embedding process of labels with the well transformed rather than the original features implicitly, where k-CMLL still outperforms other methods on the whole.

\subsection{Parameter Sensitivity Analysis}

To explore the influence of balance parameter $\alpha$ in CMLL, we fix $\mu= \nu =50\%$, $\lambda=0$ and run CMLL with $\alpha$ ranging in $\{10^{-4},10^{-3}, ... ,10^3,10^4\}$.
To be convenient, we denote $dep = tr[\bm{V}^t\bm{HXP{P}^{t}{X}^{t}H}\bm{H}\bm{V}]$ and $rec = tr[\bm{V}^{t}\bm{YY}^{t}\bm{V}]$ as the dependence term and recovery term in objective (\ref{eq:dual}).
Dropping the recovery term in (\ref{eq:dual}), we can find the solutions of $\bm{V}$ and $\bm{P}$, and then compute corresponding values of dependence term $dep_{max}$ and recovery term $rec_{min}$. Similarly, by only considering recover term in (\ref{eq:dual}), we can find the solution of $\bm{V}$ and calculate corresponding $dep_{min}$ and $rec_{max}$. To obtain a comprehensive understanding, we normalize the value of two terms by:
\begin{equation}
rec' = \frac{rec - rec_{min}}{rec_{max} - rec_{min}}, ~~
dep' = \frac{dep - dep_{min}}{dep_{max} - dep_{min}}. \nonumber
\end{equation}

The experimental results on msra with average precision are showed in Fig.~\ref{fig:alpha} as an example. The curve in Fig.~\ref{fig:alpha}~(a) indicates that setting $\alpha$ too big or too small will both result in bad performance. And it seems that an unreasonable big $\alpha$ suffers more, which indicates that an encoder with good recovery ability is very important for CMLL. Fig.~\ref{fig:alpha}~(b) provides the explanation for the curve trend in Fig.~\ref{fig:alpha}~(a).
For example, the curve of average precision drops sharply when $log_{10}(\alpha)$ changes from 1 to 2. And the reason is that $dep'$ is already very close to the upper bound when $\alpha=10$. As $\alpha$ further rises to $100$, the increment of $dep'$ is very limited while $rec'$ decreases obviously. This suggests that $\alpha$ do not need to be increased when $dep'$ is close to its maximization, otherwise the decrement of $rec'$ will result in a decline of the performance. Thus we can utilize such plot to guide the tuning process of $\alpha$. To sum up, aiming to get better performance, we should consider an appropriate trade-off between dependence term and recovery term.

\subsection{Analysis on the Compression Ratio}

We investigate how the compression ratio influence the performance, we fix one space changeless and gradually move the compression ration of the other space from 100\% to 10\%, and due to page limitation, we only display the curve of average precision on enron in Fig.~\ref{fig:enron_ap}. It can be seen that CMLL usually achieves a better or comparable performance, and other curves show similar results. The performance curves of metrics also reveal the effectiveness of the proposed method.

The reason for the superiority of CMLL is that it follows the spirit of CL. CMLL links the embedding process of the label space and the feature space to each other and guides each process by another well-disposed space. Instead, most other embedding methods either focus on the embedding of just one space, or guides the embedding process by original problematic space. Therefore, CMLL performs well especially in noisy, redundant and sparse datasets. However, the embedding may bring the loss of information when we compress the dense or non-redundant datasets into a very low dimension.

In reality, traversing every possible pair $(\mu,\nu)$ to lock the best one is unaffordable. Here we give an empirical method for that. We draw the spatial graphs of CMLL for collected datasets over various $\mu$ and $\nu$. And the spatial graphs of some datasets on average precision are displayed in Fig.~\ref{fig:3d} as examples.

It can be seen that, for different $\nu$ on each dataset, the general trends of CMLL over various $\mu$ are almost the same. Knowing this, we can first conduct CMLL with a fixed $\nu$ over various $\mu$, and lock the best $\mu^{*}$ in such situation.
Then we can run CMLL with $\mu^{*}$ over various $\nu$ to find the best $\nu^{*}$. Finally near best ratio pair is given as $(\mu^{*}, \nu^{*})$. One can also try some different starting $\nu$ to make the searching process more precise. In practice, we find that the ratio pair searched by this empirical method can achieve comparable performance to the real best one in most cases. Especially, we find that $\mu$ may have little influence on the performance in some datasets. In other words, a very small compression ratio can perform as well as other ratios in CMLL, which proves the existence of redundancy and shows the superiority of CMLL to reduce the computational and space complexities.

\section{Conclusion}\label{sec:conclude}

In this paper, we provided a different insight into the MLC for fully capturing the high-order correlation between features and labels, named compact learning. We analyzed its rationality and necessity in the situation, where the feature space suffers from redundancy or noise, and meanwhile, the label space is deteriorated by noise or sparsity - frequent occurrences in MLC. Following the spirit of compact learning, a simple yet effective method termed CMLL that is compatible with flexible multi-label classifiers was proposed. By conducting the embedding process of the features and the labels seamlessly, CMLL achieved a more compact representation for both the spaces. We demonstrated through experiments that CMLL can result in significant improvements for the multi-label classification.

As an initial effort towards compact learning, there are several potential ways that the current CMLL can be further improved: (a) Except the linear embedding or its kernel version, other encoding and decoding strategies, such as autoencoders and its extension, are worthwhile to be investigated; (b) Inspired by the manifold learning that the local topological structure can be shared between the feature manifold and the label manifold \cite{multi2016hou}, the structure information could be utilized for CMLL; (c) CMLL provides another possible solution to some weakly supervised learning problem, e.g., the missing label \cite{lv2019weakly} or noisy label \cite{patrini2017making}.

\section*{Acknowledgement}
This research was supported by the National Key Research \& Development Plan of China (No.2017YFB1002801), the National Science Foundation of China (61622203), the Jiangsu Natural Science Funds for Distinguished Young Scholar (BK20140022), the Collaborative Innovation Center of Novel Software Technology and Industrialization, and the Collaborative Innovation Center of Wireless Communications Technology.

\newpage
\bibliographystyle{elsarticle-num}
\bibliography{PR_CMLL}

\end{document}